%% file: opennet_cvpr16.tex
\documentclass[10pt,twocolumn,letterpaper]{article}

\usepackage{cvpr}
\usepackage{times}
\usepackage{epsfig}
\usepackage{graphicx}
\usepackage{algpseudocode}
\usepackage{algorithm}
\usepackage{subcaption}
\usepackage{flushend}
\usepackage{amsmath}
\usepackage{enumitem}
\usepackage{amssymb}
\usepackage{amsthm}
\newtheoremstyle{mydefstyle}{}{}{\itshape}{}{\bfseries}{:}{.5em}{#1 #2 (\thmnote{#3})}
\theoremstyle{mydefstyle}

\newtheorem{mytheory}{Theorem}

\newcommand{\argmax}{\operatornamewithlimits{argmax}}
\newcommand{\argsort}{\operatornamewithlimits{argsort}}
\long\def\advercomment#1{}
\long\def\comment#1{}

\makeatletter
\renewcommand\section{\@startsection {section}{1}{\z@}%
                                   {-1.1ex \@plus -2ex \@minus -.2ex}%
                                   {.5ex \@plus.2ex}%
                                   {\normalfont\Large\bfseries}}
\renewcommand\subsection{\@startsection{subsection}{2}{\z@}%
                                   {-1.0ex \@plus -2ex \@minus -.2ex}%
                                   {.5ex \@plus.2ex}%
	                           {\normalfont\large\bfseries}}
\renewcommand\subsubsection{\@startsection{subsubsection}{3}{\z@}%
                                     {-.5ex\@plus -.2ex \@minus -.2ex}%
                                     {.2ex \@plus .2ex}%
                                     {\normalfont\large\bfseries}}

\renewenvironment{proof}[1][\proofname]{\par
  \vspace*{-.5 \topsep}% remove the space after the theorem
  \pushQED{\qed}%
  \normalfont
  \topsep0pt \partopsep0pt % no space before
  \trivlist
  \item[\hskip\labelsep
        \itshape
    #1\@addpunct{.}]\ignorespaces
}{%
  \popQED\endtrivlist\@endpefalse
  \addvspace{2pt plus 2pt} % some space after
}

\def\tightmath{
\abovedisplayskip=4pt plus 2pt minus 1pt 
\abovedisplayshortskip=2pt plus 1pt minus 1pt 
\belowdisplayskip=4pt plus 2pt minus 1pt 
\belowdisplayshortskip=2pt plus 1pt minus 1pt }
\def\crushmath{
\abovedisplayskip=1pt plus 1pt minus 2pt 
\abovedisplayshortskip=1pt plus 1pt minus 2pt 
\belowdisplayskip=1pt plus 1pt minus 2pt 
\belowdisplayshortskip=1pt plus 1pt minus 2pt }
\tightmath
\crushmath

\makeatother

\usepackage[pagebackref=true,breaklinks=true,letterpaper=true,colorlinks,bookmarks=false]{hyperref}

\def\cvprPaperID{348} % *** Enter the CVPR Paper ID here

\cvprfinaltrue
\ifcvprfinal\fi
\begin{document}

%%%%%%%%% TITLE
\title{Towards Open Set Deep Networks}

\author{Abhijit Bendale, Terrance E. Boult\\
University of Colorado at Colorado Springs\\
%Institution1 address\\
{\tt\small \{abendale,tboult\}@vast.uccs.edu}}
% For a paper whose authors are all at the same institution,
% omit the following lines up until the closing ``}''.
% Additional authors and addresses can be added with ``\and'',
% just like the second author.
% To save space, use either the email address or home page, not both

\maketitle
%\thispagestyle{empty}

%%%%%%%%% ABSTRACT
\begin{abstract}

Deep networks have produced significant gains for various visual recognition
problems, leading to high impact academic and commercial applications.  Recent
work in deep networks highlighted that it is easy to generate images that humans
would never classify as a particular object class, yet networks classify such
images high confidence as that given class -- deep network are easily fooled
with images humans do not consider meaningful.  The closed set nature of deep
networks forces them to choose from one of the known classes leading to such
artifacts.  Recognition in the real world is open set, i.e. the recognition
system should reject unknown/unseen classes at test time. We present a
methodology to adapt deep networks for open set recognition, by introducing a
new model layer, OpenMax, which estimates the probability of an input being from
an unknown class.  A key element of estimating the unknown probability is
adapting Meta-Recognition concepts to the activation patterns in the penultimate
layer of the network. OpenMax allows rejection of ``fooling'' and unrelated open
set images presented to the system; OpenMax greatly reduces the number of {\em
  obvious errors} made by a deep network.  We prove that the OpenMax concept
provides bounded open space risk, thereby formally providing an open set
recognition solution. We evaluate the resulting open set deep networks using
pre-trained networks from the Caffe Model-zoo on ImageNet 2012 validation data,
and thousands of fooling and open set images. The proposed OpenMax model
significantly outperforms open set recognition accuracy of basic deep networks
as well as deep networks with thresholding of SoftMax probabilities.

% Visual recognition problems are generally open set and the system must reject unknown classes at test time.  While deep networks have produce significant gains for close set problems, the final SoftMax layer is inherently closed set.  The problem has recently been highlighted by  work showing one can generate images that are visually far away to humans yet  produce high confidence incorrect results on deep nets. This paper shows how to adapt deep networks for open set recognition, and that the subsequent open deep networks reject fooling and unrelated open set images.  In particular, we introduce a new layer model, OpenMax, which incorporates an estimated unknown class probability and directly support open set recognition.  We further show how to  meta-recongnition  applied to overall activation patterns in the penultimate layer of the network provides a usable estimate of unknown class probability.  We show these formally provide for bounded open space risk, and hence provide an open set recognition solution.  We evaluate  the resulting open set deep networks using networks from the Caffe Model-zoo tested on ImageNet 2012 train/validation data as well as thousands of fooling and 50,000 images open set images, from classes in ImageNet 2010 that do not occur in 2012.    The OpenMax model signficannly outperforms basic thresholing on deep net probablity. 
\end{abstract}

%\vspace*{-3mm}
%%%%%%%%% BODY TEXT
\section{Introduction}

Computer Vision datasets have grown from few hundred images to millions of images
and from few categories to thousands of categories, thanks to research advances in
 vision and learning. Recent research in deep networks has significantly improved many aspects of
visual recognition \cite{google-LeNet15, vedaldi-VGG-bmvc14, hinton-nips12}.
Co-evolution of rich representations, scalable classification methods and large datasets 
have resulted in many commercial applications \cite{google, facebook, ebay, microsoft}.
However, a wide range of operational challenges occur while deploying recognition systems in the dynamic and ever-changing real world. 
A vast majority of recognition systems are designed for a static closed world, where the 
primary assumption is that all categories are known a priori. Deep networks,  
like many classic machine learning tools, are designed to perform closed set recognition.

Recent work on open set recognition \cite{openset-pami13,openset-prob-pami14}
and open world recognition \cite{openworld_2015}, has formalized processes for
performing recognition in settings that require rejecting unknown objects during
testing. While one can always train with an ``other'' class for uninteresting
classes ({\em known unknowns}), it is impossible to train with all possible
examples of unknown objects. Hence the need arises for designing visual
recognition tools that formally account for the ``unknown unknowns''\cite{rumsfeld2011known}.  Altough a
range of algorithms has been developed to address this issue \cite{da-aaai14,
  openset-pami13, openset-prob-pami14, socher2010learning,  novety-jenawacv15},
performing open set recognition with deep networks has remained an unsolved
problem.

\begin{figure*}
\begin{center} \includegraphics[width=\textwidth]{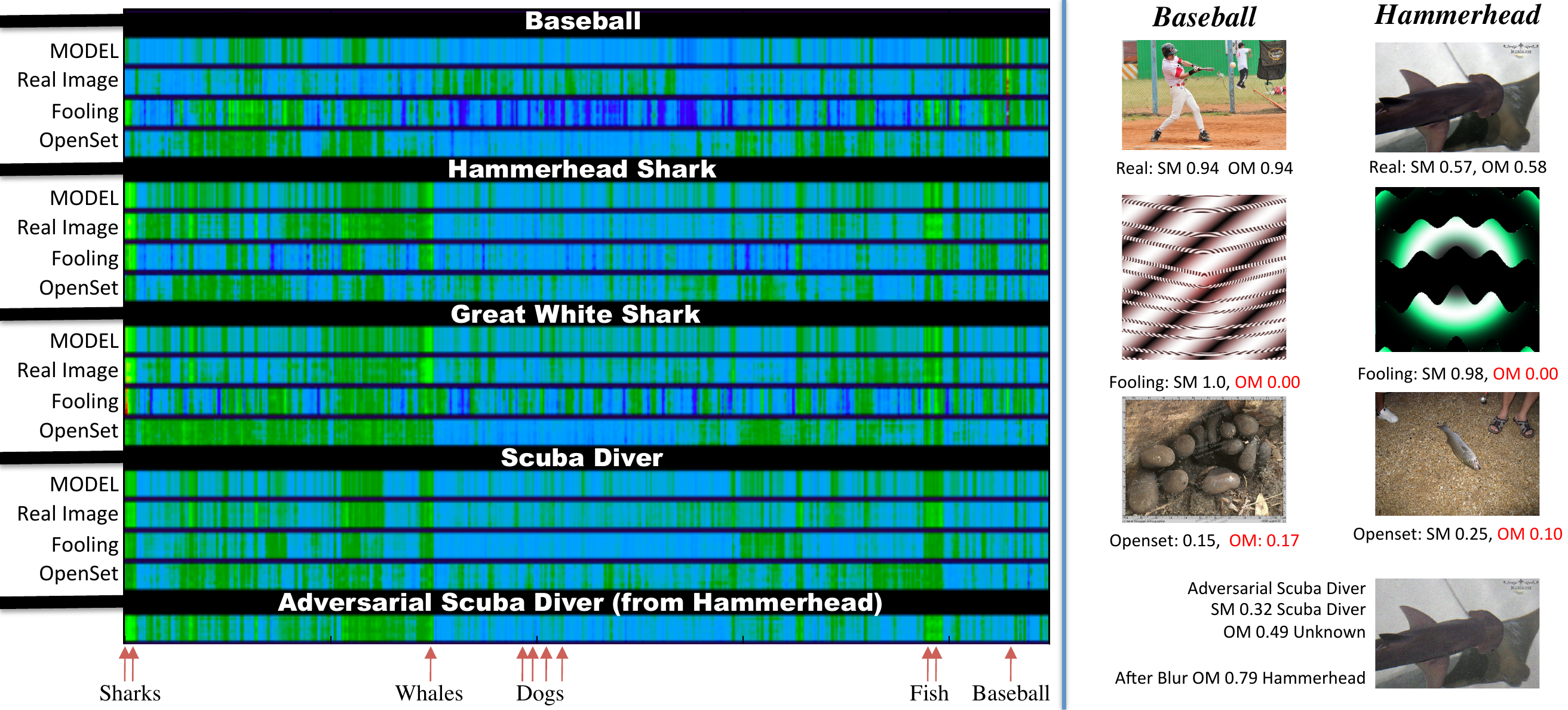} \end{center}
\vspace{-2mm}
  \caption{\small Examples showing how an activation vector model provides sufficient information for our
    Meta-Recognition and OpenMax extension of a deep network to support open-set recognition.  The OpenMax algorithm measures distance between an {\em activation vector (AV)} for an input and the model vector for the top few classes,  adjusting scores and providing an  estimate of probability of being unknown. The left side shows activation vectors (AV) for different images, with different AVs separated by black lines.  Each input image becomes an AV, displayed as 10x450 color pixels, with the vertical being one pixel for each of 10 deep network channel activation energy and the horizontal dimension showing the response for the first 450 ImageNet classes. Ranges of various category indices (sharks, whales, dogs, fish, etc.) are identified on the bottom of the image.    For each of four classes (baseball, hammerhead shark, great white shark and scuba diver), we  show an AV for 4 types of images: the model, a real  image, a fooling  image and an open set image.  The AVs show patterns of activation in which, for real images, related classes are often responding together, e.g.,  sharks share many visual features, hence correlated responses,  with other sharks, whales, large fishes,  but not with dogs or with baseballs.   Visual inspection of the AVs shows significant difference  between the response patterns for fooling and open set images compared to a real image or the model AV. For example, note the darker (deep blue) lines in many fooling images and different green patterns in many open set images. The bottom AV is from an ``adversarial'' image, wherein a hammerhead image was converted, by adding nearly invisible pixel changes, into something classified as scuba-diver.   On the right are two columns showing the associated images for two of the classes.  Each example shows the SoftMax (SM) and OpenMax (OM) scores for the  real image, the fooling and open set image that produced the AV shown on the left. The red OM scores implies the OM algorithm classified the image as unknown, but for completeness we show the OM probability of baseball/hammerhead class for which there was originally confusion.     The bottom right shows the adversarial image and its associated scores -- despite the network classifying it as a scuba diver, the visual similarity to the hammerhead is clearly stronger.  OpenMax rejects the adversarial image as an outlier from the scuba diver class. As an example of recovery from failure, we note that if the image is Gaussian blurred OpenMax classifies it as a hammerhead shark with .79 OM probability.
  }
  %% \caption{We show how distance in activation vector space provides a
  %%   meta-recognition (MR) estimate that allows extension of a
  %%   deep network to open-set recognition using OpenMax.  Adversarial images (red dots) are still closer
  %%   than fooling images (yellow squares), but both are generally farther away than
  %%   correctly classified data.  Each row of the figure shows MR scores with (+) for
  %%   correct class, red dots 100 adversarial images and yellow squares for 150
  %%   fooling images.  Also for each row we show  a black triangle for the EVT
  %%   estimated per-class threshold, green vertical bar for optimal empirical threshold
  %% }
\label{fig:teaser}
\end{figure*}

In the majority of deep networks \cite{hinton-nips12, google-LeNet15, vedaldi-VGG-bmvc14}, the output of the last fully-connected layer is fed to the SoftMax function, which produces a probability distribution over the N known class labels.  While a deep network will always have a most-likely class, one might hope that for an unknown input all classes would have low probability and that thresholding on uncertainty would reject unknown classes.  Recent papers have shown how to produce ``fooling'' \cite{nguyen2015deep} or ``rubbish'' \cite{goodfellow-iclr15} images that are visually far from the desired class but produce high-probability/confidence scores.  They strongly suggests that thresholding on uncertainty is not sufficient to determine what is unknown.  In Sec.~\ref{sec:expts}, we show that extending deep networks to threshold SoftMax probability improves open set recognition somewhat, but does not resolve the issue of fooling images.  Nothing in the theory/practice of deep networks, even with thresholded probabilities, satisfies the formal definition of open set recognition offered in \cite{openset-pami13}.  This leads to the first question addressed in this paper, {\em ``how to adapt deep networks support to open set recognition?''}

The SoftMax layer is a significant component of the problem because of its closed nature.  
We propose an alternative, OpenMax, which extends SoftMax layer by
 enabling it to predict an unknown class. OpenMax incorporates likelihood of the recognition system failure. This likelihood is used to
 estimate the probability for a given input belonging to an unknown class.
For this estimation, we adapt the concept of Meta-Recognition\cite{metarecognition-pami11,parikh-2014, ferrari-mr-eccv-12} to deep networks. We use the scores from the penultimate layer of deep networks (the fully connected layer before SoftMax, e.g., FC8) to estimate if the input is ``far'' from known training data. We call scores in that layer the {\em activation vector}(AV).
This information is incorporated in our OpenMax model and used to characterize failure of recognition system. 
By dropping the restriction for the probability for known classes to sum to 1, and rejecting inputs far from known inputs, OpenMax can formally handle unknown/unseen classes during operation.  Our experiments demonstrate that the proposed combination of OpenMax and Meta-Recognition ideas readily address open set recognition for deep networks and reject high confidence fooling images \cite{nguyen2015deep}.

While fooling/rubbish images are, to human observers, clearly not from a class of interest, adversarial images \cite{goodfellow-iclr15, Fergus-iclr14} present a more difficult challenge.  These adversarial images are visually indistinguishable from a training sample but are designed so that deep networks produce high-confidence but incorrect answers. This is different from standard open space risk because adversarial images are ``near'' a training sample in input space, for any given output class.  

A key insight in our opening deep networks is noting that ``open space risk''
should be measured in feature space, rather than in pixel space. In prior work,
open space risk is not measured in pixel space for the majority of problems
\cite{openset-pami13,openset-prob-pami14,openworld_2015}.  Thus, we ask ``is
there a feature space, ideally a layer in the deep network, where these
adversarial images are {\em far away} from training examples, i.e., a layer
where unknown, fooling and adversarial images become outliers in an open set
recognition problem?'' In Sec.~\ref{sec:mav}, we investigate the choice of the
feature space/layer in deep networks for measuring open space risk.  We show
that an extreme-value meta-recognition inspired distance normalization process
on the overall activation patterns of the penultimate network layer provides a
rejection probability for OpenMax normalization for unknown images, fooling
images and even for many adversarial images. In Fig.~\ref{fig:teaser}, we show
examples of activation patterns for our model, input images, fooling images,
adversarial images (that the system can reject) and open set images.

In summary the contributions of this paper are:
%\vspace{-2pt}
\begin{enumerate}[noitemsep,nolistsep]
  \setlength{\itemsep}{0pt}
  \setlength{\parskip}{0pt}
\item Multi-class Meta-Recognition using Activation Vectors to
    estimate the probability of deep network failure
%\item Generalization of SoftMax to OpenMax to support open set recognition
\item Formalization of open set deep networks using Meta-Recognition and
  OpenMax, along with the proof showing that proposed approach manages open space risk for deep networks
\item Experimental analysis of the effectiveness of open set deep networks at
  rejecting unknown classes, fooling images and obvious errors from adversarial images, 
  while maintaining its accuracy on testing images
\end{enumerate}

\section{Open Set Deep Networks}

A natural approach for opening a deep network is to apply a threshold on the
output probability.  We consider this as rejecting uncertain predictions, rather
than rejecting unknown classes. It is expected images from unknown classes will
all have low probabilities, i.e., be very uncertain.  This is true only for a
small fraction of unknown inputs.  Our experiments in Sec.~\ref{sec:expts} show
that thresholding uncertain inputs helps, but is still relatively weak tool for
open set recognition.  Scheirer \etal \cite{openset-pami13} defined open space
risk as the risk associated with labeling data that is ``far'' from known
training samples.  That work provides only a general definition and does not
prescribe how to measure distance, nor does it specify the space in which such
distance is to be measured. In order to adapt deep networks to handle open set
recognition, we must ensure they manage/minimize their open space risk and have
the ability to reject unknown inputs.

Building on the concepts in \cite{openset-prob-pami14, openworld_2015}, we seek
to choose a layer (feature space) in which we can build a compact abating
probability model that can be thresholded to limit open space risk. We develop
this model as a decaying probability model based on distance from a learned
model.  In following section, we elaborate on the space and meta-recognition
approach for estimating distance from known training data, followed by a
methodology to incorporate such distance in decision function of deep
networks. We call our methodology OpenMax, an alternative for the SoftMax
function as the final layer of the network.  Finally, we show that the overall
model is a compact abating probability model, hence, it satisfies the definition
for an open set recognition.

\subsection{Multi-class Meta-Recognition}
\label{sec:mav}

Our first step is to determine when an input is likely not from a known class,
i.e., we want to add a meta-recognition algorithm \cite{metarecognition-pami11,
  parikh-2014} to analyze scores and recognize when deep networks are likely
incorrect in their assessment. Prior work on meta-recognition used the final
system scores, analyzed their distribution based on Extreme Value Theory (EVT)
and found these distributions follow Weibull distribution.  Although one might 
use the per class scores independently and consider their distribution using
EVT, that would not produce a compact abating probability because the fooling
images show that the scores themselves were not from a compact space close to
known input training data. Furthermore, a direct EVT fitting on the set of class post
recognition scores (SoftMax layer) is not meaningful with deep networks, because
the final SoftMax layer is intentionally renormalized to follow a logistic
distribution. Thus, we analyze the penultimate layer, which is generally viewed
as a per-class estimation. This per-class estimation is converted by SoftMax
function into the final output probabilities.   

We take the approach that the network values from penultimate
layer (hereafter the {\em Activation Vector (AV)}), are not an independent
per-class score estimate, but rather they provide a distribution of what classes
are ``related.''  In Sec.~\ref{s:walk} we discuss an illustrative example based
on Fig.~\ref{fig:teaser}.

\begin{algorithm}[t]
  \begin{algorithmic}[1]
     \Require FitHigh function from libMR 
     \Require {Activation levels in the penultimate network layer
       $\mathbf{v(x)}={v_1(x) \ldots v_N(x)}$}
     \Require {For each class $j$ let  $S_{i,j} = v_j(x_{i,j})$
       for each correctly classified training example $x_{i,j}$.}
     \For{$j=1 \ldots N $}
     \State {\textbf{Compute mean AV}, $\mu_j =mean_i(S_{i,j})$}
     \State {\textbf{EVT Fit} $\rho_j = (\tau_j,\kappa_j, \lambda_j)=$ FitHigh$(\|\hat{S}_{j} - \mu_j\|,\eta)$}
     \EndFor
     \State {\textbf{Return} means $\mu_j$ and libMR models $\rho_j$}
	\end{algorithmic}
  \caption{\small EVT Meta-Recognition Calibration for Open Set Deep Networks, with per class Weibull fit to $\eta$ largest distance to mean activation  vector. 
Returns libMR models $\rho_j$ which includes parameters $\tau_i$ for shifting the data as well as the Weibull shape and scale parameters:$\kappa_i$, $\lambda_i$.} 
 \label{alg:EVT}
\end{algorithm}

Our overall EVT meta-recognition algorithm is summarized in Alg.~\ref{alg:EVT}.
To recognize outliers using AVs, we adapt the concepts of Nearest Class Mean 
\cite{hastie-PNAS02, mensink-eccv12} or  Nearest Non-Outlier \cite{openworld_2015}
and apply them per class within the activation vector, as a first
approximation.  While more complex models, such as nearest class multiple
centroids (NCMC) \cite{mensink2013distance} or NCM forests \cite{ristin-ncmf},
could provide more accurate modeling, for simplicity this paper focuses on
just using a single mean.  Each class is represented as a point, a {\em mean
  activation vector (MAV)} with the mean computed over only the correctly
classified training examples (line 2 of Alg.~\ref{alg:EVT}).

Given the MAV and an input image, we measure distance between them. We could
directly threshold distance, e.g.,  use the cross-class validation approach of
\cite{openworld_2015} to determine an overall maximum distance threshold.  In
\cite{openworld_2015}, the features were subject to metric learning to normalize
them, which makes a single shared threshold viable. However,
the lack of uniformity in the AV for different classes presents a greater challenge and, hence,
we seek a per class meta-recognition model.  
In particular, on line 3 of Alg.~\ref{alg:EVT}  we use the
libMR \cite{metarecognition-pami11} FitHigh function to do Weibull fitting on the largest of the distances
between  all correct positive training  instances and the associated $\mu_i$. This results in
a parameter $\rho_i$, which is used to estimate the probability of an input being an outlier with respect to class $i$.

Given $\rho_i$, a simple rejection model would be for the user to define a
threshold that decides if an input should be rejected, e.g., ensuring 90\% of
all training data will have probability near zero of being rejected as an
outlier.  While simple to implement, it is difficult to calibrate an absolute Meta-Recognition 
threshold because it depends on the unknown unknowns.  Therefore, we choose to use this in the OpenMax algorithm described
in Sec.~\ref{alg:OpenMax} which has a continuous adjustment. 

We note that our calibration process uses only correctly classified data, for
which class $j$ is rank 1. At testing, for input $\mathbf{x}$ assume class $j$
has the largest probability, then $\rho_j(\mathbf{x})$ provides the MR estimated
probability that $\mathbf{x}$ is an outlier and should be rejected.  We use one
calibration for high-ranking (e.g., top 10), but as an extension separate
calibration for different ranks is possible.  Note when there are multiple
channels per example we compute per channel per class mean vectors
$\mu_{j,c}$ and Weibull parameters $\rho_{j,c}$.  It is worth remembering that
{\em the goal is not to determine the training class of the input, rather this
  is a meta-recognition process used to determine if the given input is from an
  unknown class and hence should be rejected}.

\subsection{Interpretation of Activation Vectors}
\label{s:walk}
In this section, we present the concept of activation vectors and
meta-recognition with illustrative examples based on Fig.~\ref{fig:teaser}.

\textbf{Closed Set:} Presume the input is a valid input of say a hammerhead
shark, i.e., the second group of activation records from Fig.~\ref{fig:teaser}.
The activation vector shows high scores for the AV dimension associated with a
great white shark.  All sharks share many direct visual features and many
contextual visual features with other sharks, whales and large fish, which is
why Fig.~\ref{fig:teaser} shows multiple higher activations (bright yellow-green)
for many ImageNet categories in those groups.  We hypothesize that for most
categories, there is a relatively consistent pattern of related activations.
The MAV captures that distribution as a single point. The AVs present a space
where we measure the distance from an input image in terms of the activation of
each class; if it is a great white shark we also expect higher activations from
say tiger and hammerhead sharks as well as whales, but very weak or no
activations from birds or baseballs.  Intuitively, this seems like
the right space in which to measure the distance during training.

\textbf{Open Set:} First let us consider an open set image, i.e., a real image
from an unknown category.  These will always be mapped by the deep network to
the class for which SoftMax provides the maximum response, e.g.,  the images of
rocks in Fig.~\ref{fig:teaser} is mapped to baseball and the
fish on the right is mapped to a hammerhead.  Sometimes open set images will
have lower confidence, but the maximum score will yield a corresponding class.
Comparing the activation vectors of the input with the MAV for a class for which
the input produced maximum response, we observe it is often far from the mean.
However, for some open set images the response provided is close  to the AV but
still has an overall low activation level.  This can occur if the input is an
``unknown'' class that is closely related to a known class, or if the object is
small enough that it is not well distinguished.  For example, if the input is
from a different type of shark or large fish, it may provide a low activation,
but the AV may not be different enough to be rejected.  For this reason, it is
still necessary for open set recognition to threshold uncertainty, in addition
to directly estimating if a class is unknown.  

\textbf{Fooling Set:} Consider a fooling input image, which was artificially
constructed to make a particular class (e.g.,  baseball or hammerhead) have high
activation score and, hence, to be detected with high confidence.  While the
artificial construction increases the class of interest's probability, the image generation process
 did not simultaneously adjust the 
scores of all related classes, resulting in an AV that is ``far'' from the model AV.
Examine the 3rd element of each class group in Fig.~\ref{fig:teaser} which show activations from fooling
images.  Many fooling images are
visually quite different and so are their activation vectors.  The many regions
of very low activation (dark blue/purple) are likely because one can increase
the output of SoftMax for a given class by reducing the activation of other
classes, which in turn reduces the denominator of the SoftMax computation.

\textbf{Adversarial Set:} Finally, consider an adversarial input image
\cite{goodfellow-iclr15, Fergus-iclr14, Yosinski-icml15DL}, which is
constructed to be close to one class but is mislabeled as another.  An example
is shown on the bottom right of Fig.~\ref{fig:teaser}.  If the adversarial image
is constructed to a nearby class, e.g., from hammerhead to great white, then the
approach proposed herein will fail to detect it as a problem -- fine-grained
category differences are not captured in the MAV.  However, adversarial images
can be constructed between any pair of image classes, see \cite{Fergus-iclr14}.
When the target class is far enough, e.g., the hammerhead and scuba example
here, or even farther such as hammerhead and baseball, the adversarial image
will have a significant difference in activation score and hence can be rejected.
We do not consider adversarial images in our experiments because the outcome
would be more a function of that adversarial images we choose to generate --
and we know of no meaningful distribution for that. If, for example, we choose
random class pairs $(a,b)$ and generated adversarial images from $a$ to $b$,
most of those would have large hierarchy distance and likely be rejected.  If we
choose the closest adversarial images, likely from nearby classes, the
activations will be close and they will not be rejected.

The result of our OpenMax process is that open set as well as fooling or
adversarial images will generally be rejected. Building a fooling or adversarial
image that is not rejected means not only getting a high score for the class of
interest, it means maintaining the relative scores for the 999 other classes. At
a minimum, the space of adversarial/fooling images is significantly reduced by
these constraints. Hopefully, any input that satisfies all the constraints is an
image that also gets human support for the class label, as did some of the
fooling images in Figure 3 of \cite{nguyen2015deep}, and as one sees in
adversarial image pairs fine-grain separated categories such as bull and great
white sharks.

One may wonder if a single MAV is sufficient to represent complex objects with
different aspects/views.  While future work should examine more complex models
that can capture different views/exemplars, e.g., NCMC
\cite{mensink2013distance} or NCM forests \cite{ristin-ncmf}. If the deep
network has actually achieved the goal of view independent recognition, then the
distribution of penultimate activation should be nearly view independent.  While
the open-jaw and side views of a shark are visually quite different, and a
multi-exemplar model may be more effective in capturing the different features
in different views, the open-jaws of different sharks are still quite similar,
as are their side views.  Hence, each view may present a relatively consistent
AV, allowing a single MAV to capture both.  Intuitively, while image features
may vary greatly with view, the relative strength of ``related classes''
represented by the AV should be far more view independent.

\subsection{OpenMax}
\label{s:OpenMax}
The standard SoftMax function is a gradient-log-normalizer of the categorical
probability distribution -- a primary reason that it is commonly used as the last
fully connected layer of a network.  The traditional definition has per-node
weights in their computation. The scores in the penultimate network layer of
Caffe-based deep networks \cite{jia2014caffe}, what we call the activation
vector, has the weighting performed in the convolution that produced it.  Let
$\mathbf{v(x)}={v_1(x),\ldots, v_N(x)}$ be the activation level for each class,
$y=1,\ldots, N$.
After deep network training, an input image $\mathbf{x}$ yields activation
vector $\mathbf{v(x)}$,  the SoftMax layer computes:
\begin{equation}
P(y=j|\mathbf{x}) = \frac{e^{\mathbf{v_j(x)}}}{\sum_{i=1}^N  e^{\mathbf{v_i(x)}}}
\label{eq:SoftMax}
\end{equation}
where the denominator sums over all classes to ensure the probabilities over all
classes sum to 1.  However, in open set recognition there are unknown classes
that will occur at test time and, hence, it is not appropriate to require the
probabilities to sum to 1.

\begin{algorithm}[t]
  \begin{algorithmic}[1]
     \Require{Activation vector for $\mathbf{v(x)}={v_1(x),\ldots, v_N(x)}$}
     \Require {\textbf{means} $\mu_j$ and libMR models $\rho_j=(\tau_i,\lambda_i,\kappa_i)$}
     \Require {$\alpha$, the numer of ``top'' classes to revise}
     \State {Let $s(i) = \argsort (v_j(x))$; Let $\omega_{j}=1$ }
     \For{$i=1,\ldots, \alpha $}
     \State{$\omega_{s(i)}(x) =1-\frac{\alpha-i}{\alpha} e^{-\left(\frac{\|x-\tau_{s(i)}\|}{\lambda_{s(i)}}\right)^{\kappa_{s(i)}}}$}
     \EndFor
     \State{Revise activation vector $\hat{v}(x) = \mathbf{v(x)} \circ \mathbf{\omega(x)}$  }
     \State{Define $\hat{v}_0(x) = \sum_i v_i(x)(1-\omega_i(x))$.}
     \State{\begin{equation}\hat P(y=j|\mathbf{x}) = \frac{e^{\mathbf{\hat v_j(x)}}}{\sum_{i=0}^N  e^{\mathbf{\hat v_i(x)}}}\label{eq:OpenMax}\end{equation}}
     \State{Let  $y^*=\argmax_j P(y=j|\mathbf{x})$}
     \State{Reject input if $y^*==0$ or $P(y=y^*|\mathbf{x}) < \epsilon$}
  \end{algorithmic}

  \caption{\small OpenMax probability estimation with rejection of unknown or uncertain inputs.}
 \label{alg:OpenMax}
\end{algorithm}

To adapt SoftMax for open set, let $\rho$ be a vector of meta-recognition models
for each class estimated by Alg.~\ref{alg:EVT}.  In Alg.~\ref{alg:OpenMax} we
summarize the steps for OpenMax computation.  For convenience we define the {\em
  unknown unknown} class to be at index $0$.  We use the Weibull CDF probability
(line 3 of Alg.~\ref{alg:OpenMax}) on the distance between $\mathbf{x}$ and
$\mu_i$ for the core of the rejection estimation.  The model $\mu_i$ is computed
using the images associated with category $i$, images that were classified correctly
(top-1) during training process.  We expect the EVT function of distance to
provide a meaningful probability only for few top ranks.  Thus in line 3 of
Alg.~\ref{alg:OpenMax}, we compute weights for the $\alpha$ largest activation
classes and use it to scale the Weibull CDF probability. We then compute revised
activation vector with the top scores changed. We compute a pseudo-activation
for the unknown unknown class, keeping the total activation level constant.
Including the unknown unknown class, the new revised activation compute the
OpenMax probabilities as in Eq.~\ref{eq:OpenMax}.

OpenMax provides probabilities that support explicit rejection 
when the unknown unknown class ($y=0$) has the largest probability.  This Meta-Recognition
approach is a first step toward determination of unknown unknown classes and our
experiments show that a single MAV works reasonably well at detecting fooling images,
and is better than just thresholding on uncertainty.  However, in any system
that produces certainty estimates, thresholding on uncertainty is still a valid
type of meta-recognition and should not be ignored.  The final OpenMax approach
thus also rejects unknown as well as uncertain inputs in line 9 of
Alg.\ref{alg:OpenMax}.

 \begin{figure}[t]
\centering

\includegraphics[width=1.0\columnwidth]{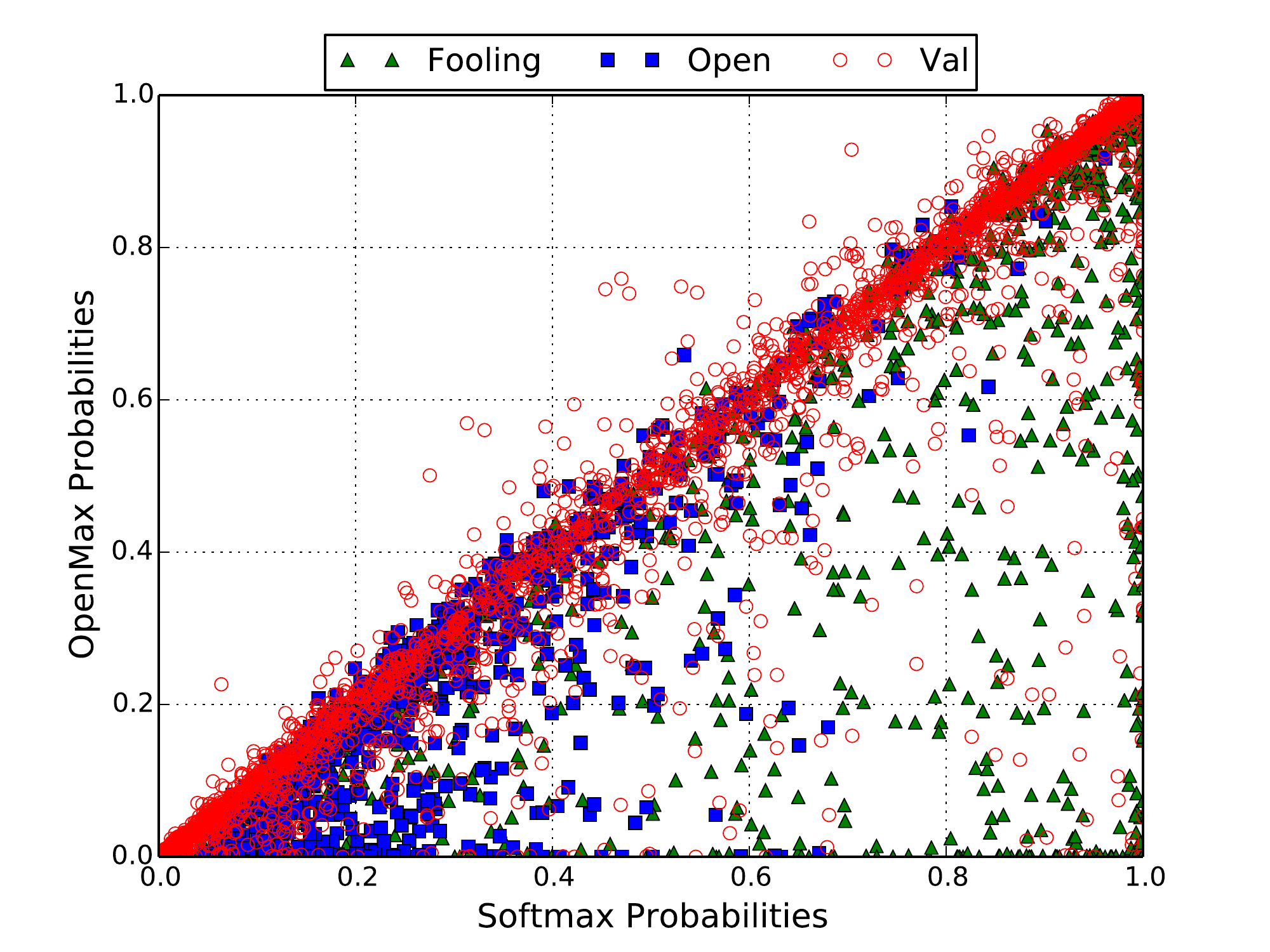}
\caption{\small A plot of OpenMax probabilities vs SoftMax probabilities for the
  fooling (triangle), open set (square) and validation (circle) for 100
  categories from ImageNet 2012.  The more off-diagonal a point, the more
  OpenMax altered the probabilities. Below the diagonal means OpenMax estimation
  reduced the inputs probability of being in the class.  For some inputs OpenMax
  increased the classes probability, which occurs when the leading
  class is partially rejected thereby reducing its probability and increasing a
  second or higher ranked class.  Uncertainty-based rejection threshold
  ($\epsilon$) selection can optimize F-measure between correctly classifying
  the training examples while rejecting open set examples. (Fooling images are
  not used for threshold selection.)  The number of triangles and squares below
  the diagonal means that uncertainty thresholding on OpenMax threshold
  (vertical direction), is better than thresholding on SoftMax (horizontal
  direction).}
\label{fig:probplot}
\end{figure}

To select the hyper-parameters $\epsilon,\eta,$ and $\alpha$, we can do a grid
search calibration procedure using a set of training images plus a sampling of
open set images, optimizing F-measure over the set.  The goal here is basic
calibration for overall scale/sensitivity selection, not to optimize the
threshold over the space of {unknown unknowns}, which cannot be done
experimentally.

Note that the computation of the unknown unknown class probability inherently alters 
all probabilities estimated. For a fixed threshold and inputs that have even a small chance of being unknown, OpenMax will reject more inputs than SoftMax. Fig.~\ref{fig:probplot} shows the OpenMax and SoftMax probabilities
for 100 example images, 50 training images and 50 open set images as well as for
fooling images. The more off-diagonal the more OpenMax altered the
probabilities.  Threshold selection for uncertainty based rejection $\epsilon$,
would find a balance between keeping the training examples while rejecting open
set examples. Fooling images were not used for threshold selection.  

While not part of our experimental evaluation, note that OpenMax also provides
meaningful rank ordering via its estimated probability. Thus OpenMax directly
supports a top-5 class output with rejection.  It is also important to note that
because of the re-calibration of the activation scores $\hat{v}_i(x)$, OpenMax
often does not produce the same rank ordering of the scores.

 \subsection{OpenMax Compact Abating Property}

While thresholding uncertainty does provide the
ability to reject some inputs, it has not been shown to formally limit open
space risk for deep networks.  It should be easy to see that in terms of the
activation vector, the positively labeled space for SoftMax is not restricted to
be near the training space, since any increase in the maximum class score
increases its probability while decreasing the probability of other
classes. With sufficient increase in the maximum directions, even large changes in other
dimension will still provide large activation for the leading class.  While in
theory one might say the deep network activations are bounded, the fooling
images of \cite{nguyen2015deep}, are convincing evidence that SoftMax cannot
manage open space risk.

 \begin{mytheory}[Open Set Deep Networks]
A deep network extended using Meta-Recognition on activation vectors as in
Alg.~\ref{alg:OpenMax}, with the SoftMax later adapted to OpenMax, as in
Eq.~\ref{eq:OpenMax}, provides an open set recognition function.
 \end{mytheory}
 \begin{proof} 
The Meta-Recognition probability (CDF of a Weibull) is a monotonically increasing function of
$\|\mu_i-x\|$, and hence $1-\omega_i(x)$ is monotonically decreasing.  Thus,
they form the basis for a compact abating probability as defined in
\cite{openset-prob-pami14}.  Since the OpenMax transformation is a weighted
monotonic transformation of the Meta-Recognition probability, applying Theorems 1 and 2 of
\cite{openworld_2015} yield that thresholding the OpenMax probability of the
unknown manages open space risk as measured in the AV feature space.  Thus it is
an open set recognition function.
\end{proof}

\section{Experimental Analysis} \label{sec:expts}

In this section, we present experiments carried out in order to evaluate the effectiveness of the
proposed OpenMax approach for open set recognition tasks with deep neural networks.  Our evaluation is based on
ImageNet Large Scale Visual Recognition Competition (ILSVRC) 2012 dataset with 1K visual categories.
The dataset contains around 1.3M images for training (with approximately 1K to 1.3K images per 
category), 50K images for validation and 150K images for testing. Since test labels for ILSVRC 2012 are
not publicly available, like others have done we report performance on validation set \cite{hinton-nips12, nguyen2015deep,
karen-iclr15}. We use a pre-trained AlexNet (BVLC AlexNet) deep neural network
provided by the Caffe software package \cite{jia2014caffe}. BVLC AlexNet is reported to obtain
approximately 57.1\% top-1 accuracy on ILSVRC 2012 validation set. The choice of pre-trained BVLC AlexNet
is deliberate, since it is open source and one of the most widely used packages available for deep learning.
 
To ensure proper open set evaluation, we apply a test protocol similar to the ones
presented in \cite{openset-prob-pami14, openworld_2015}. During the testing phase,
we test the system with all the 1000 categories from ILSVRC 2012 validation
set, fooling categories and previously unseen categories. The
previously unseen categories are selected from ILSVRC 2010. It has been noted by
Ruskovsky \etal \cite{ILSVRC15} that approximately 360 categories from ILSVRC
2010 were discarded and not used in ILSVRC 2012.   Images from these 360
categories as the {\em open set} images, i.e., unseen or unknown categories.

Fooling images are generally totally unrecognizable to humans as belonging to
the given category but deep networks report with near certainty they are from
the specified category. We use fooling images provided by Nguyen \etal
\cite{nguyen2015deep} that were generated by an evolutionary algorithm or by gradient
ascent in pixel space.  The final test set consists of 50K closed set images
from ILSVRC 2012, 15K open set images (from the 360 distinct categories from ILSVRC
2010) and 15K fooling images (with 15 images each per ILSVRC 2012 categories).

\textbf{Training Phase:} As discussed previously (Alg.~\ref{alg:EVT}), we
consider the penultimate layer (fully connected layer 8 , i.e., \textit{FC8})
for computation of mean activation vectors (MAV). The MAV vector is computed for
each class by considering the training examples that deep networks training
classified correctly for the respective class.  MAV is computed for each
crop/channel separately. Distance between each correctly classified training
example and MAV for particular class is computed to obtain class specific
distance distribution. For these experiments we use a distance that is a
weighted combination of normalized Euclidean and cosine distances.  Supplemental
material shows results with pure Euclidean and other measures that overall
perform similarly.  Parameters of Weibull distribution are estimated on these
distances. This process is repeated for each of the 1000 classes in ILSVRC
2012. The exact length of tail size for estimating parameters of Weibull
distribution is obtained during parameter estimation phase over a small set of
hold out data. This process is repeated multiple times to obtain an overall tail
size of 20.

 \begin{figure}[t]
\centering

\includegraphics[width=1.0\columnwidth]{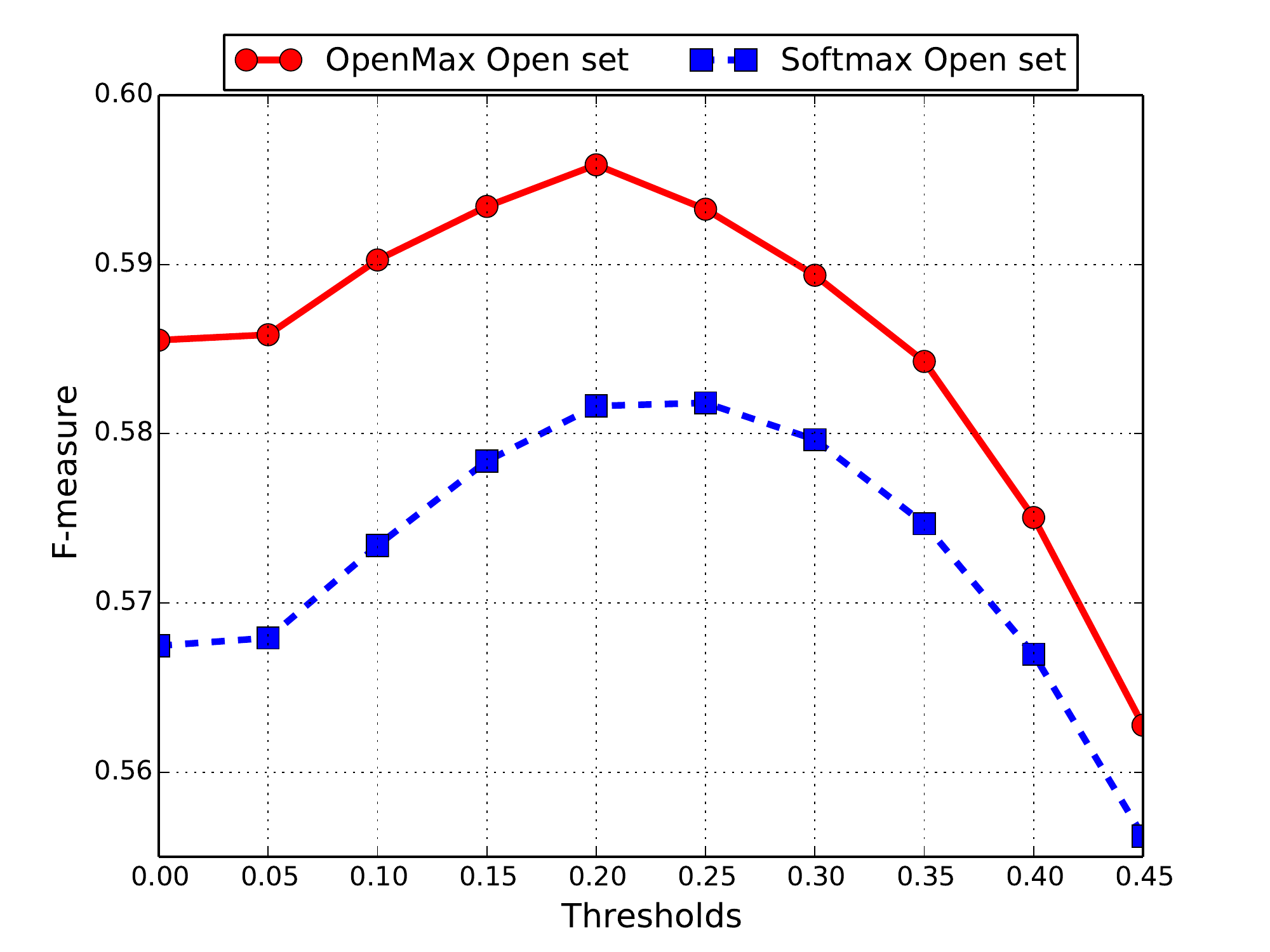}
\caption{\small OpenMax and SoftMax-w/threshold performance shown as F-measure
  as a function of threshold on output probabilities. The test uses 80,000
  images, with 50,000 validation images from ILSVRC 2012, 15,000 fooling images
  and 15,000 ``unknown'' images draw from ILSVRC 2010 categories not used in
  2012. The base deep network performance would be the same as threshold 0
  of SoftMax-w/threshold. OpenMax performance gain is nearly 4.3\%
  improvement accuracy over SoftMax with optimal threshold, and 12.3\% over the base deep network.
  Putting that in context, over the test set OpenMax correctly
  classified 3450 more images than SoftMax and 9847 more than the base deep network. }
\label{fig:fmeasureplot}
\end{figure} 
 
\textbf{Testing Phase:} During testing, each test image goes through the
OpenMax score calibration process as discussed previously in
Alg.~\ref{alg:OpenMax}. The activation vectors are the values in the
\textit{FC8} layer for a test image that consists of 1000x10 dimensional values
corresponding to each class and each channel.  For each channel in each class,
the input is compared using a per class MAV and per class Weibull parameters.
During testing, distance with respect to the MAV is computed and revised OpenMax
activations are obtained, including the new unknown class (see lines 5\&6 of
Alg.~\ref{alg:OpenMax}).  The OpenMax probability is computed per channel, using
the revised activations (Eq.~\ref{eq:OpenMax}) yielding an output of 1001x10
probabilities. For each class, the average over the 10 channel gives the overall
OpenMax probability.  Finally, the class with the maximum over the
1001 probabilities is the predicted class. This maximum probability is then
subject to the uncertainty threshold (line 9).  In this work we focus on strict
top-1 predictions.

 \begin{figure}[t]
\centering
\includegraphics[width=.9\columnwidth]{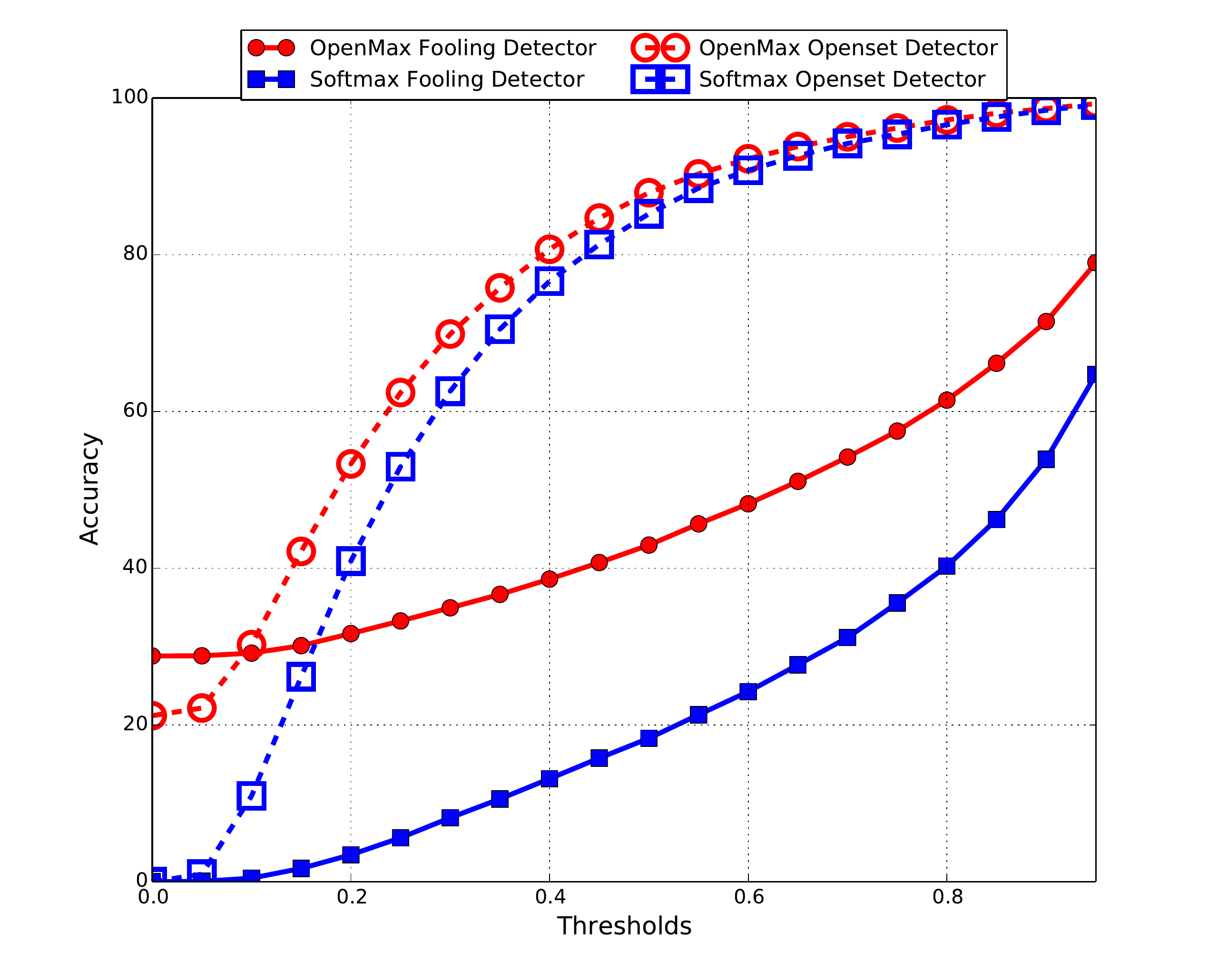}
\caption{\small The above figure shows performance of OpenMax and SoftMax as a
  detector for fooling images and for open set test images. F-measure is
  computed for varying thresholds on OpenMax and SoftMax probability values. The
  proposed approach of OpenMax performs very well for rejecting fooling images
  during prediction phase.}
\label{fig:foolplot}
\end{figure}

\textbf{Evaluation:} ILSVRC 2012 is a large scale multi-class classification
problem and top-1 or top-5 accuracy is used to measure the effectiveness of a
classification algorithm \cite{ILSVRC15}. Multi-class classification error for a
closed set system can be computed by keeping track of incorrect classifications.
For open set testing the evaluation must keep track of the errors that occur due
to standard multi-class classification over known categories as well as errors
between known and unknown categories. As suggested in
\cite{socher2010learning,openset-pami13} we use F-measure to evaluate open set
performance.  \comment{F-measure is can be computed as the harmonic mean of
  precision and recall which in terms of True Positives (TP), False Positives
  (FP) and False Negatives (FN) is given by:
\begin{equation}
\text{F-measure} = \frac{2TP}{2TP + FP + FN}.
\end{equation}}
For open
set recognition testing, F-measure is better than accuracy because it is not
inflated by true negatives.

For a given threshold on OpenMax/SoftMax probability values, we compute true
positives, false positives and false negatives over the entire dataset. For example,
when testing the system with images from validation set, fooling set and open
set (see Fig.~\ref{fig:fmeasureplot}), true positives are defined as the correct
classifications on the validation set, false positives are incorrect classifications
on the validation set and false negatives are images from the fooling set and open set
categories that the system incorrectly classified as known examples. Fig.~\ref{fig:fmeasureplot} shows performance of
OpenMax and SoftMax for varying thresholds. Our experiments show that the
proposed approach of OpenMax consistently obtains higher F-measure on open set
testing.

%dataset, model, features, fc8
%open set, fooling
% computing model activation vector, weibull calibration
% computing of F-measure, SoftMax with threshold, OpenMax

\section{Discussion}

We have seen that with our OpenMax architecture, we can automatically reject
many unknown open set and fooling images as well as rejecting some adversarial
images, while having only modest impact to the true classification rate. One
of the obvious questions when using Meta-Recognition is ``what do we do with
rejected inputs?''  While that is best left up to the operational system
designer, there are multiple possibilities. OpenMax can be treated as a novelty
detector in the scenario presented open world recognition \cite{openworld_2015}
after that human label the data and the system incrementally learn new
categories. Or detection can used as a flag to bring in other
modalities\cite{socher2013zero,frome2013devise}.

A second approach, especially helpful with adversarial or noisy images, is to
try to remove small noise that might have lead to the miss classification.  For
example, the bottom right of Fig.~\ref{fig:teaser}, showed an adversarial image
wherein a hammerhead shark image with noise was incorrectly classified by base
deep network as a scuba diver.  OpenMax Rejects the input, but with a small
amount of simple gaussian blur, the image can be reprocessed and is accepted as a
hammerhead shark by with probability 0.79.

We used non-test data for parameter tuning, and for brevity only showed
performance variation with respect to the uncertainty threshold shared by both
SoftMax with threshold and OpenMax. The supplemental material shows variation of
a wider range of OpenMax parameters, e.g. one can increase open set and fooling
rejection capability at the expense of rejecting more of the true classes. In future work, such
increase in true class rejection might be mitigated by increasing the expressiveness of the AV
model, e.g. moving to multiple MAVs per class.  This might allow it to better
capture different contexts for the same object, e.g. a baseball on a desk has a
different context, hence, may have different ``related'' classes in the AV than
say a baseball being thrown by a pitcher.

 \begin{figure}[t]
\centering
\includegraphics[width=\columnwidth]{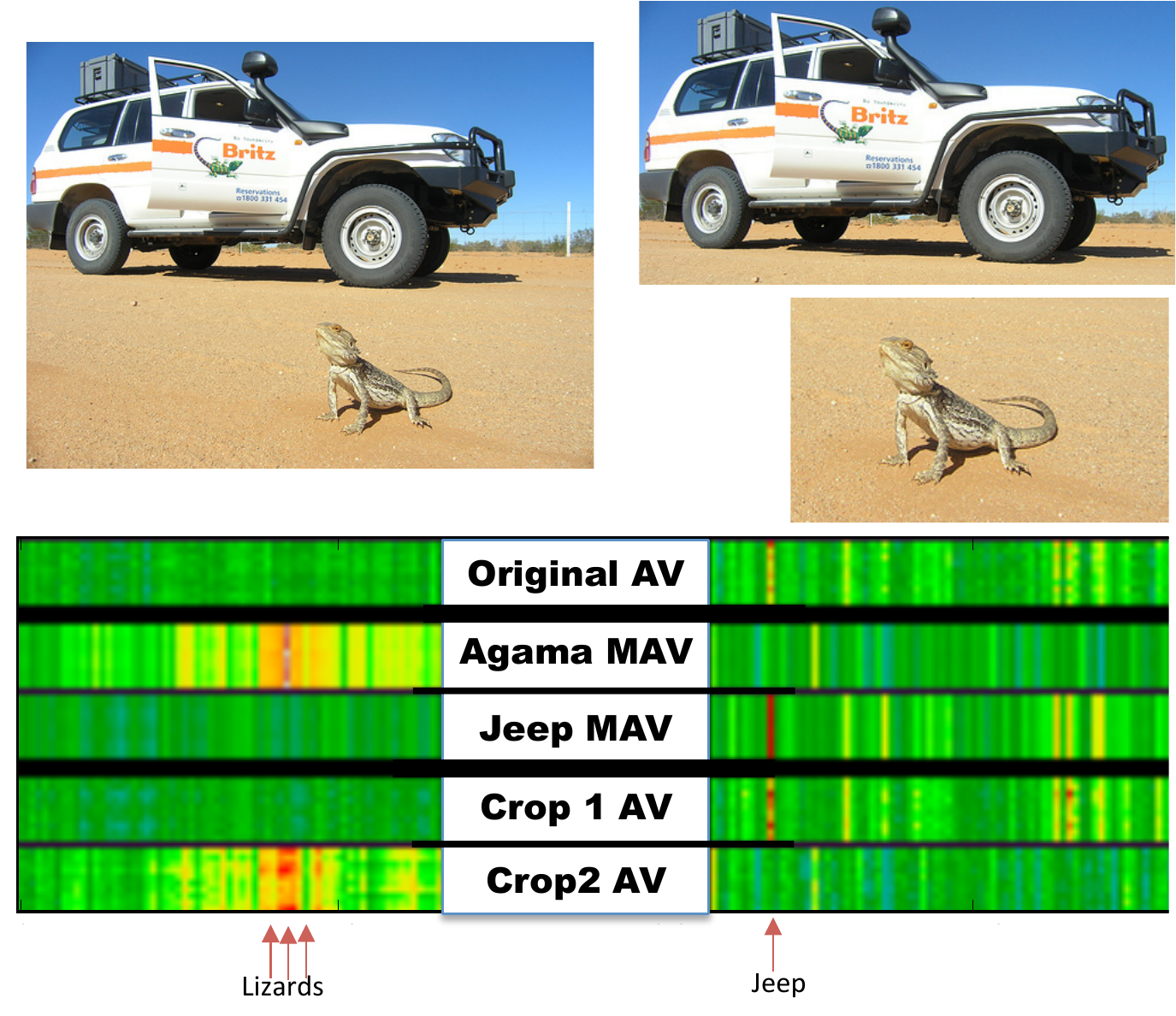}
\caption{\small OpenMax also predict failure during training as in this example.  The official class is agama but the MAV for agama is rejected for this input, and the highest scoring class is jeep with probability 0.26.  However, cropping out image regions can find windows where the agama is well detected and another where the Jeep is detected.  Crop 1 is the jeep region, crop 2 is agama and the crops AV clearly match the appropriate model and are accepted with probability 0.32 and 0.21 respectively. }
\label{f:agamma}
\end{figure}

Interestingly, we have observe that the OpenMax rejection process often
identifies/rejects the ImageNet images that the deep network incorrectly
classified, especially images with multiple objects. Similarly, many samples that 
are far away from training data have multiple objects in the scene.
Thus, other uses of the OpenMax rejection can be to improve training process and aid
in developing better localization techniques \cite{ferrari-object-bmvc15, sivic-local-cvpr15}. See Fig.~\ref{f:agamma} for an example.

\input opennet_cvpr16_supp.tex

{\small
\bibliographystyle{ieee}
\bibliography{opennet_cvpr16}
}

\end{document}

%% file: opennet_cvpr16_supp.tex
\section{Towards Open Set Deep Networks: Supplemental}

In this supplement, we provide we provide additional material to further the reader’s
understanding of the work on Open Set Deep Networks, Mean Activation Vectors,
Open Set Recognition and OpenMax algorithm. We present additional experiments on ILSVRC 2012 dataset.
First we present experiments to illustrate performance of OpenMax for various parameters
of EVT calibration (alg. 1, main paper) followed by sensitivity of OpenMax to total number of ``top classes''
(i.e. $\alpha$ in alg. 2, main paper) to consider for recalibrating SoftMax scores. We then
present different distance measures namely euclidean and cosine distance used for EVT calibration.
We then illustrate working of OpenMax with qualitative examples for open set evaluation performed
during the testing phase. Finally, we illustrate the distribution of Mean Activation Vectors with a class
confusion map.

\section{Parameters for OpenMax Calibration}

\subsection{Tail Sizes for EVT Calibration}

In this section we present extended analysis of effect of tail sizes used for EVT fitting in Alg 1 in main paper on the performance
of the proposed OpenMax algorithm. We tried multiple tail sizes for estimating parameters of Weibull distribution (line 3, Alg 1, main paper).
We found that as the tail size increased, OpenMax algorithm became very robust at rejecting images from open set and fooling set.
OpenMax continued to perform much better than SoftMax in this setting. The results of this experiments are presented
in Fig ~\ref{fig:tailsizes}. However, beyond tail size 20, we saw performance drop on validation set. This phenomenon 
can be  seen in Fig ~\ref{fig:tailsizes-fmeasures}, since F-Measure obtained on OpenMax starts to drop
beyond tail size 20. Thus, there is an optimal balance to be maintained between rejecting images from open set and fooling set, while
maintaining correct classification rate on validation set of ILSVRC 2012. 

\begin{figure*}
        \centering
        \begin{subfigure}[b]{0.3\textwidth}
                \includegraphics[width=\textwidth]{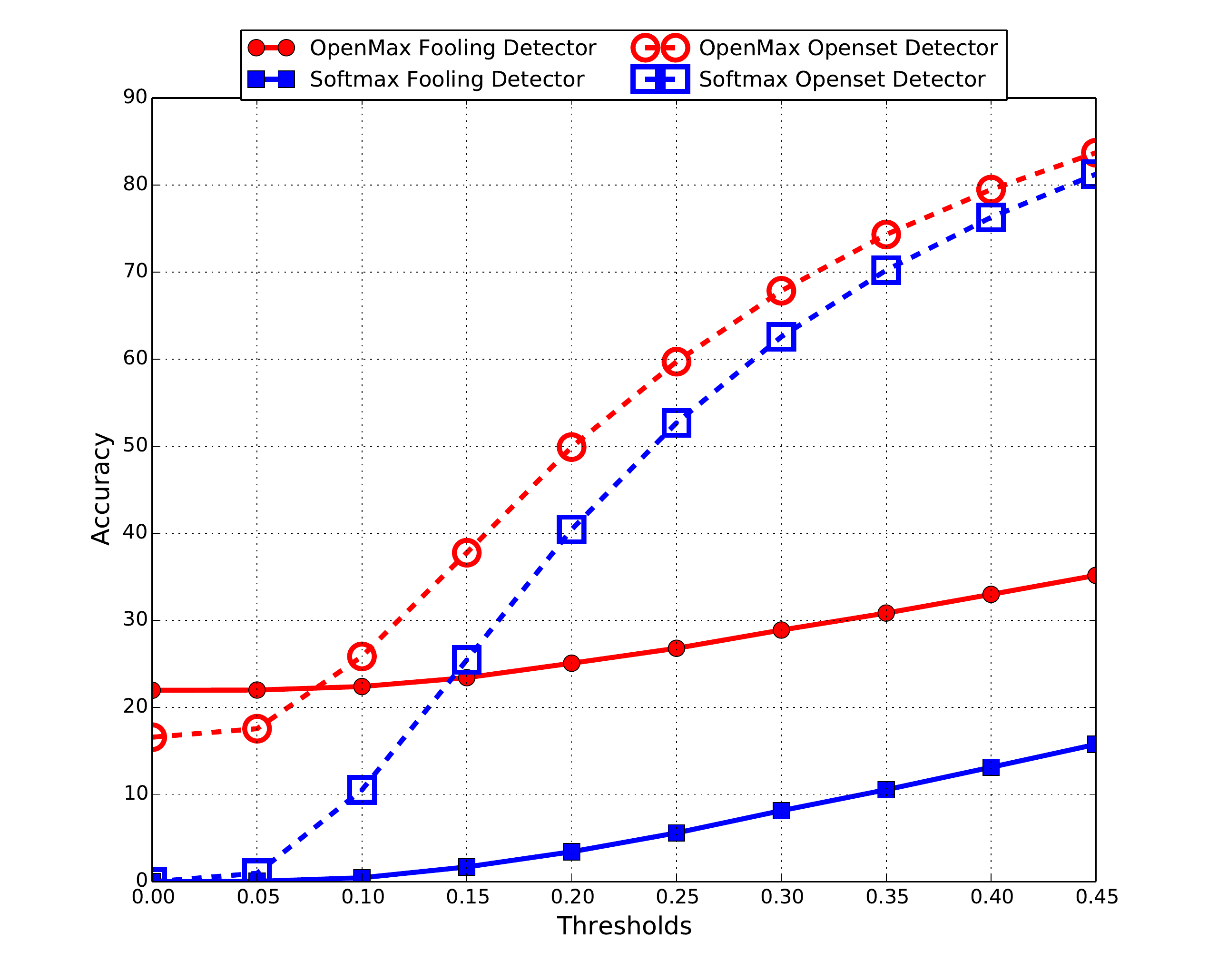}
                \caption{Tail Size 10}
                                 \vspace{20pt}
        \end{subfigure}%
        ~ %add desired spacing between images, e. g. ~, \quad, \qquad, \hfill etc.
          %(or a blank line to force the subfigure onto a new line)
        \begin{subfigure}[b]{0.3\textwidth}
                \includegraphics[width=\textwidth]{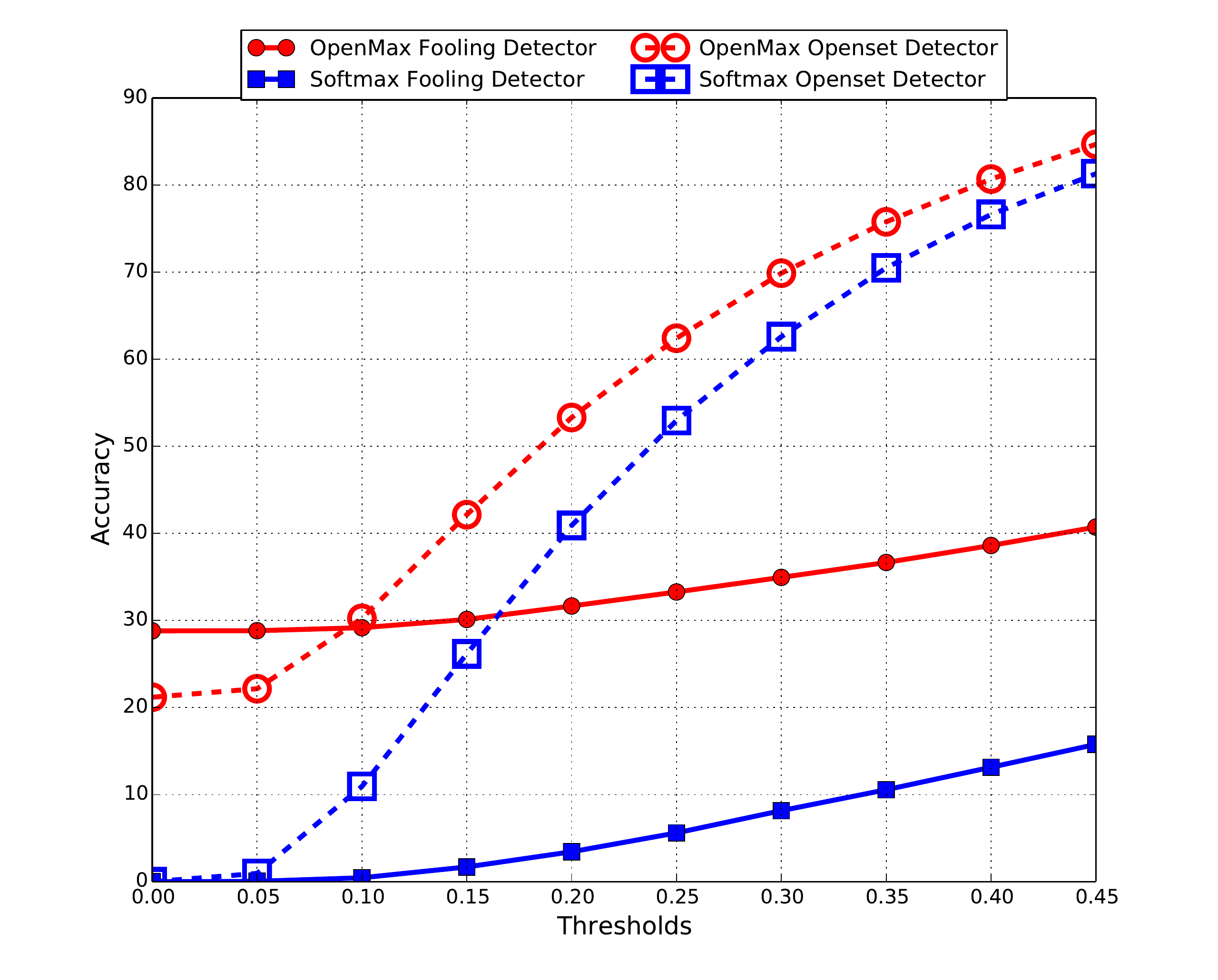}
                \caption{Tail Size 20 (optimal)}
                                 \vspace{20pt}
        \end{subfigure}
        ~ %add desired spacing between images, e. g. ~, \quad, \qquad, \hfill etc.
          %(or a blank line to force the subfigure onto a new line)
        \begin{subfigure}[b]{0.3\textwidth}
                \includegraphics[width=\textwidth]{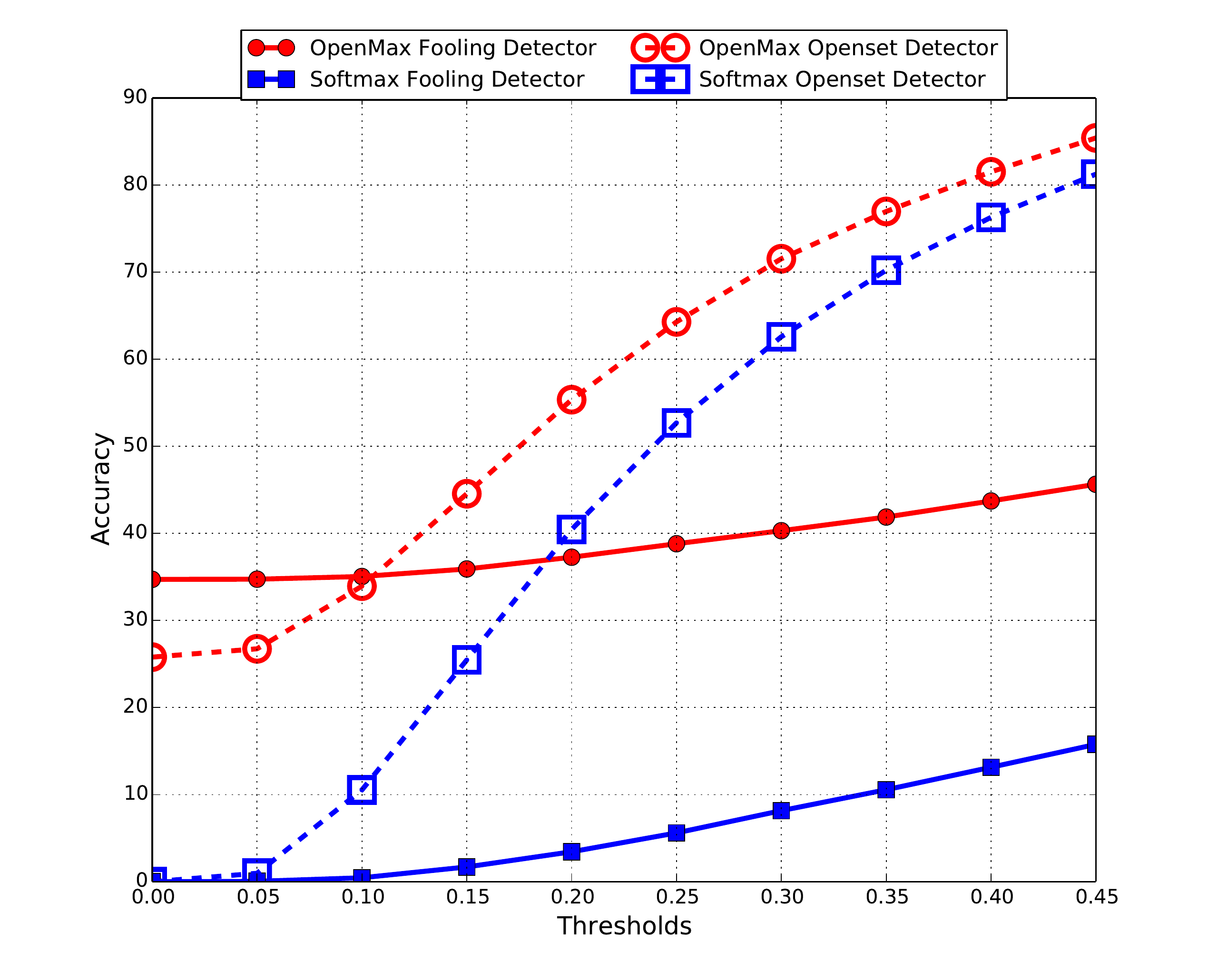}
                \caption{Tail Size 25}
                                 \vspace{20pt}
        \end{subfigure}
        
        \begin{subfigure}[b]{0.3\textwidth}
                \includegraphics[width=\textwidth]{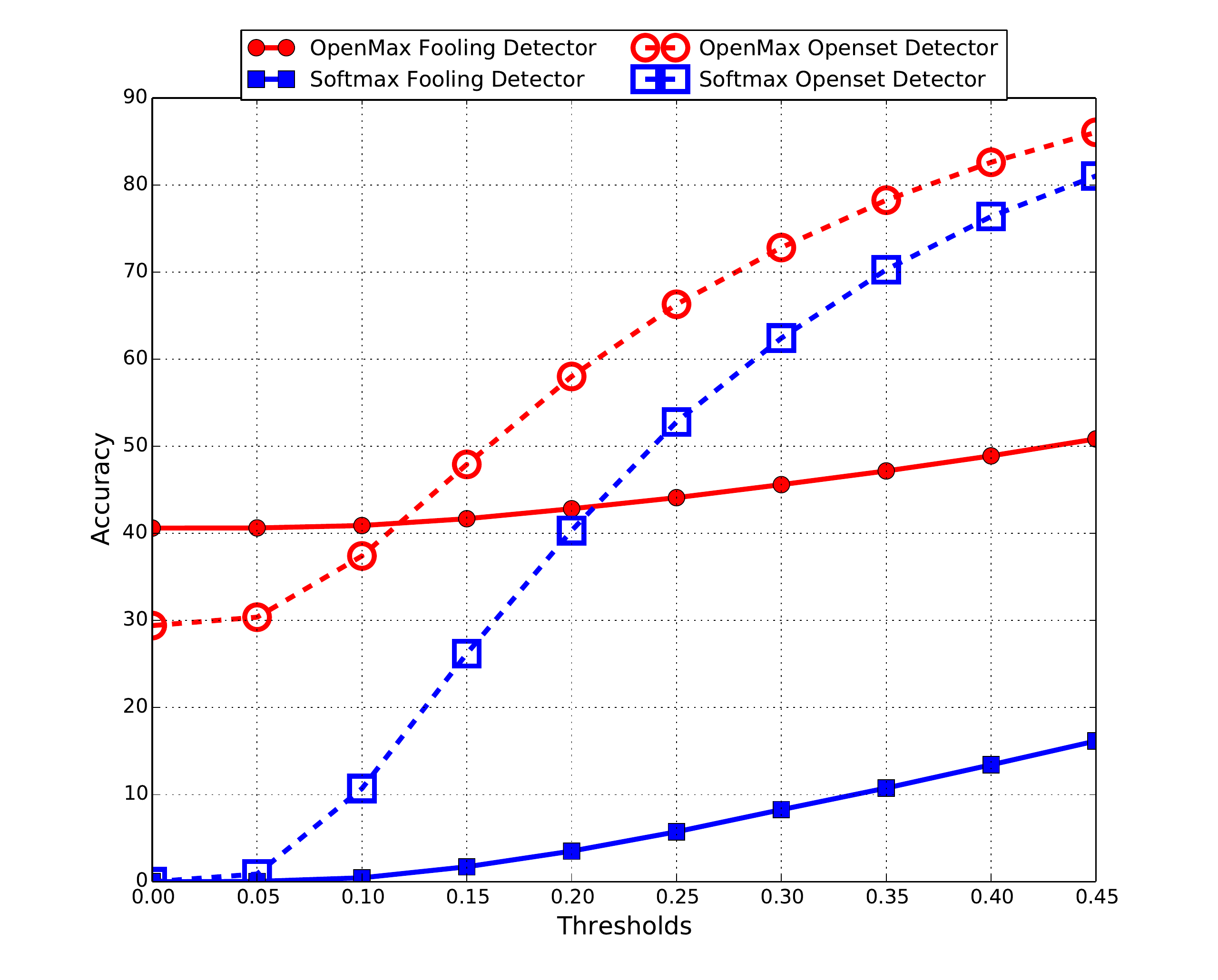}
                \caption{Tail Size 30}
        \end{subfigure}
        ~
        \begin{subfigure}[b]{0.3\textwidth}
                \includegraphics[width=\textwidth]{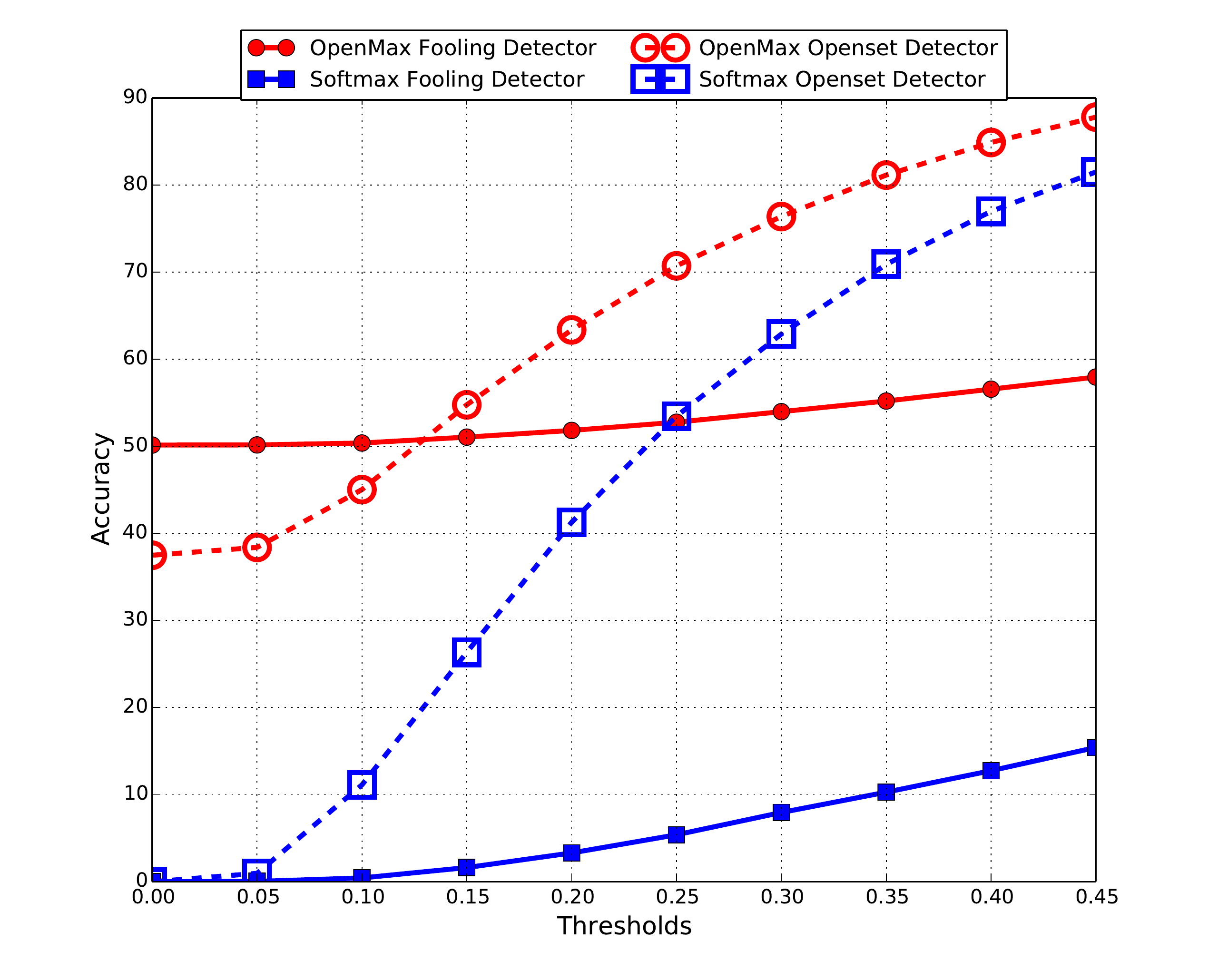}
                \caption{Tail Size 40}

        \end{subfigure}
        ~
        \begin{subfigure}[b]{0.3\textwidth}
                \includegraphics[width=\textwidth]{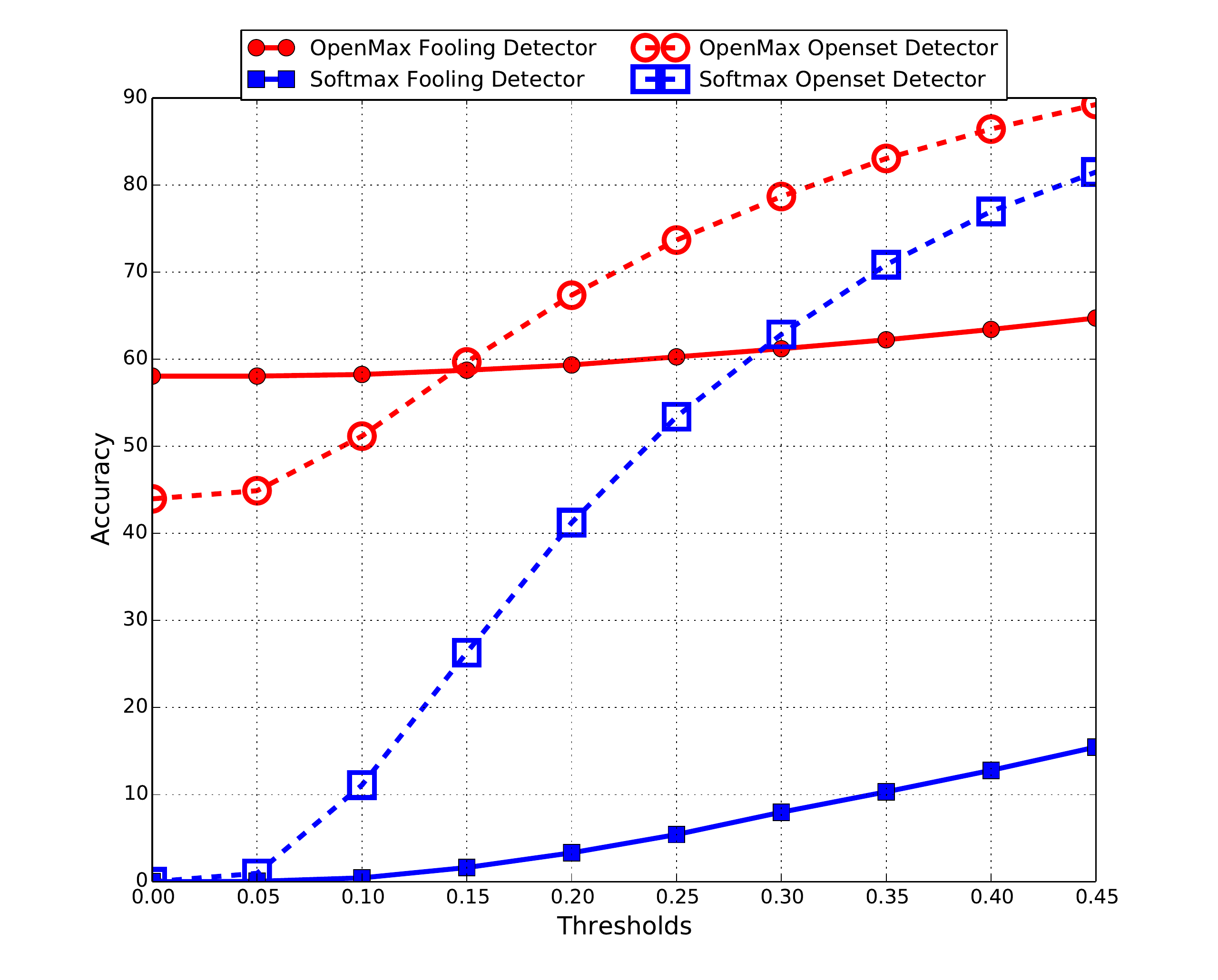}
                \caption{Tail Size 50}
        \end{subfigure}
        ~
        \caption{The graphs shows fooling detection accuracy and open set
          detection accuracy for varying tail sizes of EVT fitting. The graphs plot accuracy vs
          varying uncertainty threshold values, with different tails in each graph. We observe
          that OpenMax consistently performs better than SoftMax for varying
          tail sizes. However, while increasing tail size increases OpenMax
          rejections for open set and fooling, it also increases rejection for
          true images thereby reducing accuracy on validation set as well, see
          Fig ~\ref{fig:tailsizes-fmeasures}. These type of accuracy plots are
          often problematic for open set testing which is why in Fig
          ~\ref{fig:tailsizes-fmeasures} we use F-measure to better balance
          rejection and true acceptance. In the main paper, tail size of 20 was
          used for all the experiments.}
        \label{fig:tailsizes}
\end{figure*}

\begin{figure*}
        \centering
        \begin{subfigure}[b]{0.3\textwidth}
                \includegraphics[width=\textwidth]{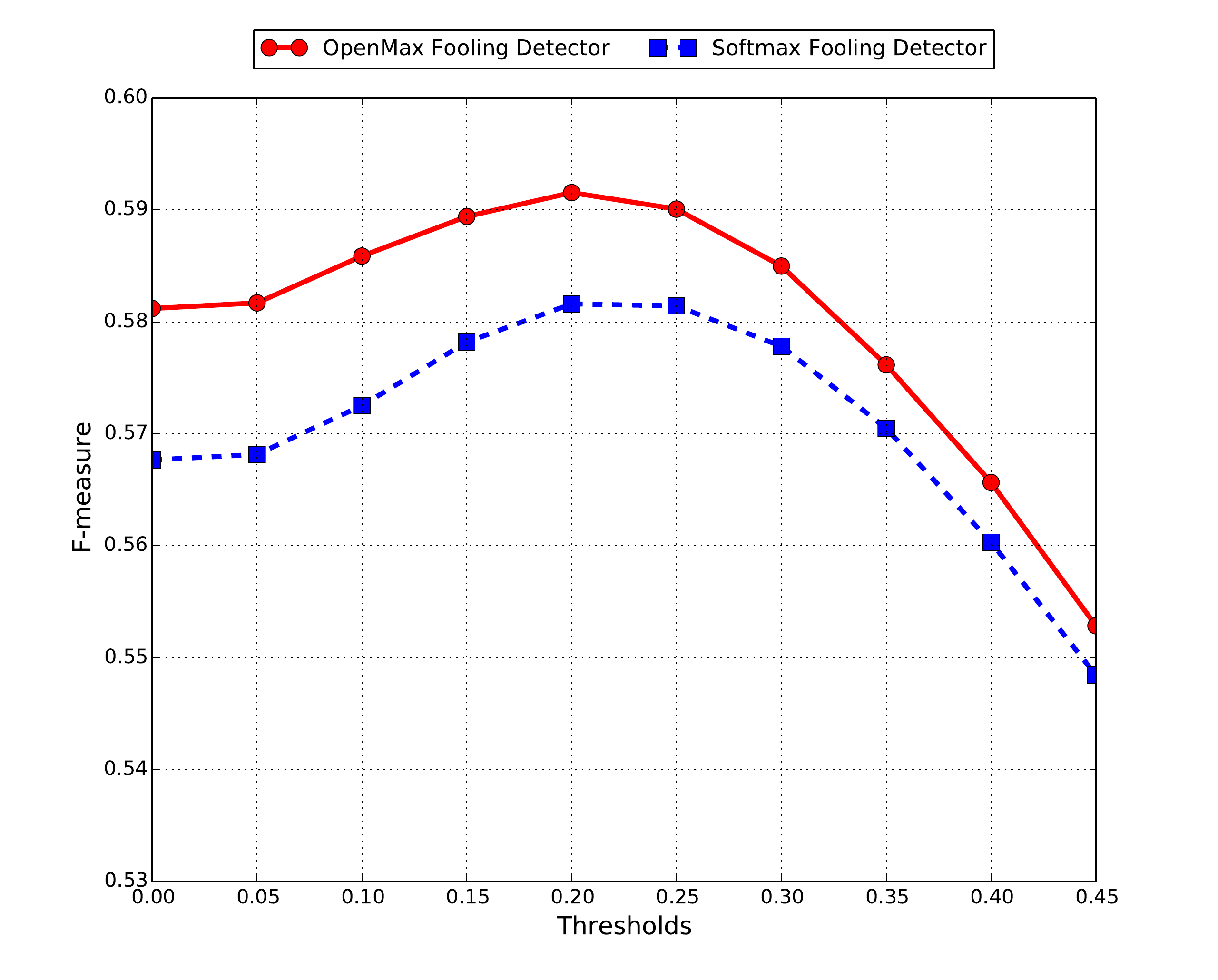}
                \caption{Tail Size 10}
                 \vspace{20pt}
        \end{subfigure}%
        ~ %add desired spacing between images, e. g. ~, \quad, \qquad, \hfill etc.
          %(or a blank line to force the subfigure onto a new line)
        \begin{subfigure}[b]{0.3\textwidth}
                \includegraphics[width=\textwidth]{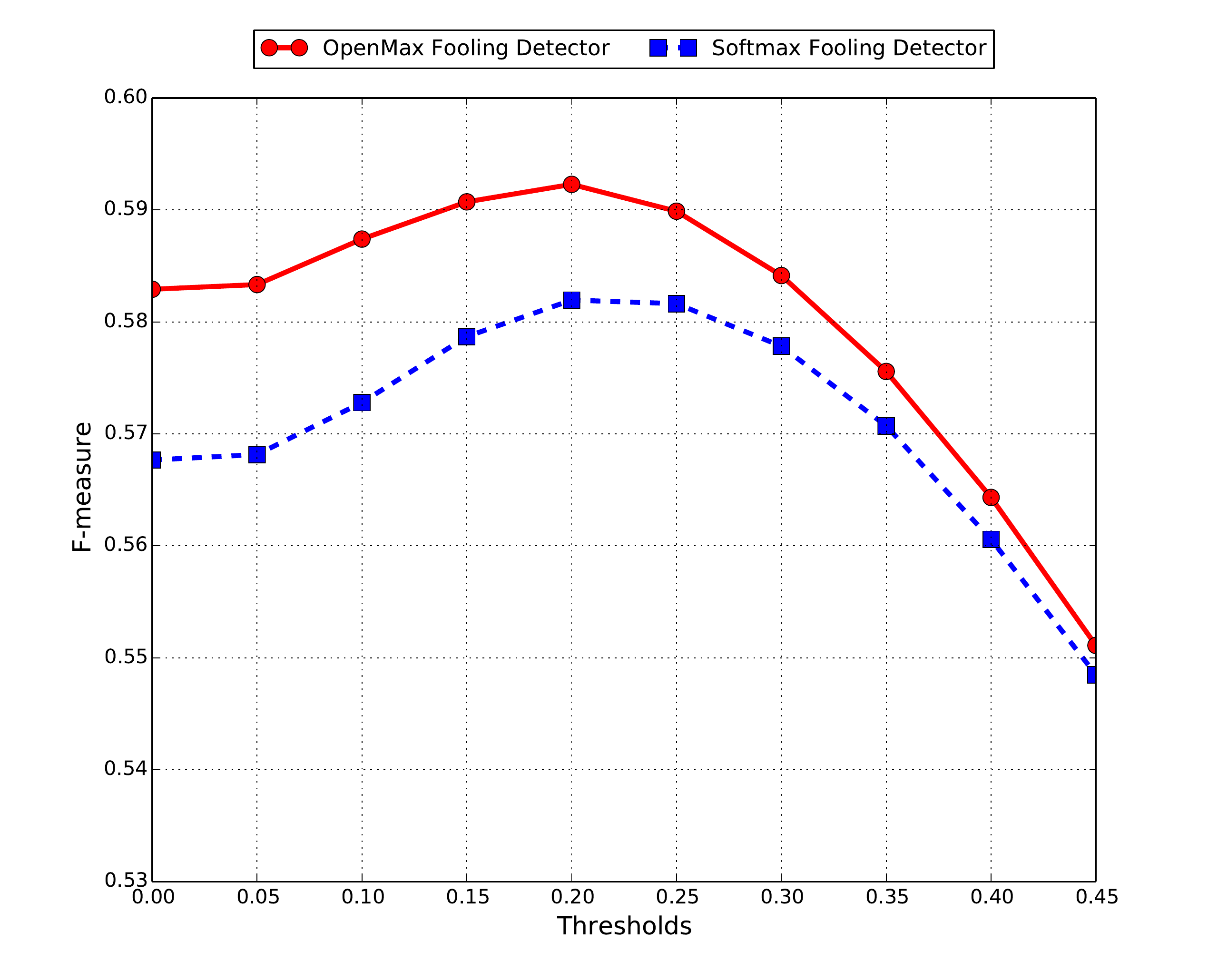}
                \caption{Tail Size 20 (optimal)}
	\vspace{20pt}
        \end{subfigure}
        ~ %add desired spacing between images, e. g. ~, \quad, \qquad, \hfill etc.
          %(or a blank line to force the subfigure onto a new line)
        \begin{subfigure}[b]{0.3\textwidth}
                \includegraphics[width=\textwidth]{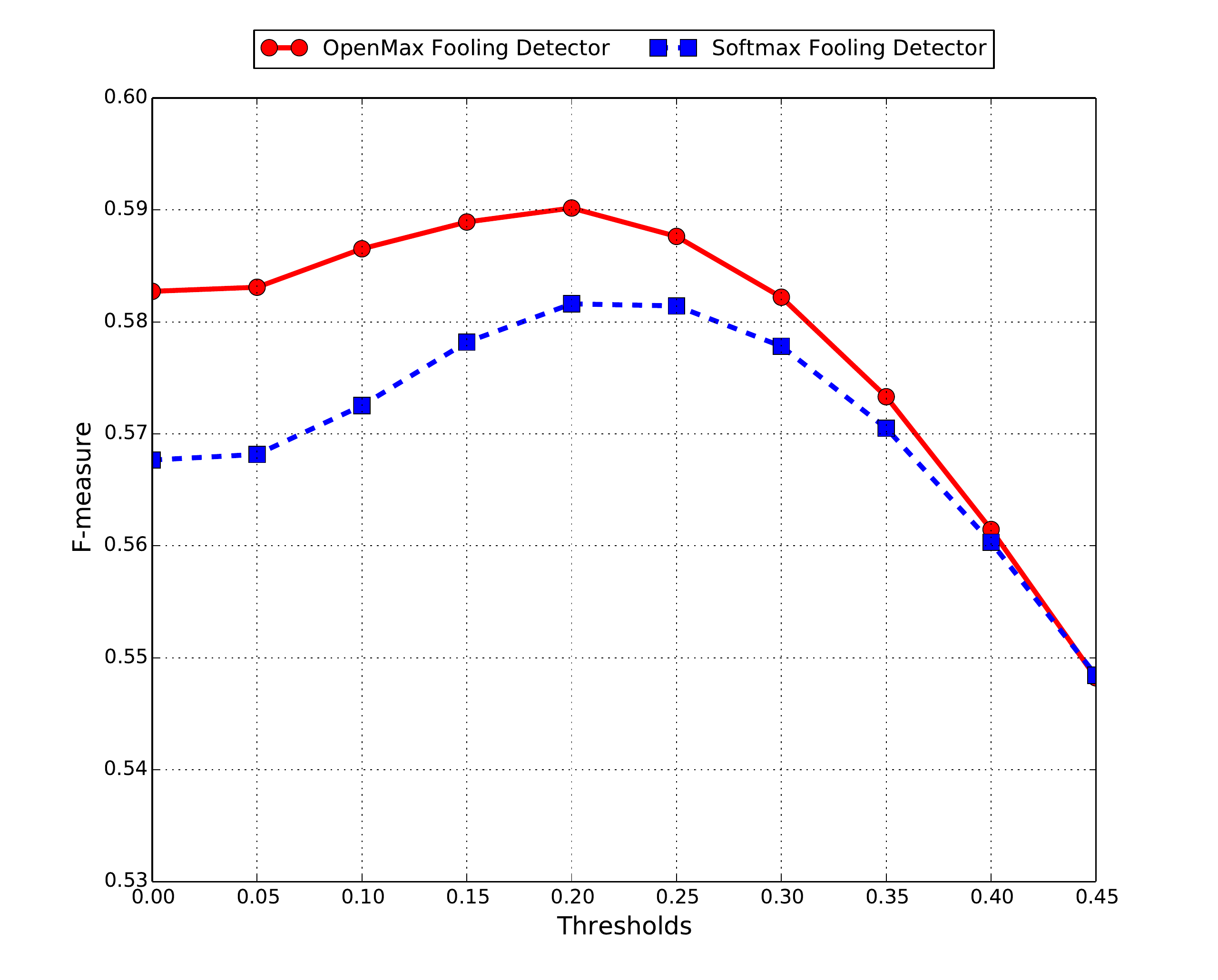}
                \caption{Tail Size 25}
                        \vspace{20pt}
        \end{subfigure}

        \begin{subfigure}[b]{0.3\textwidth}
                \includegraphics[width=\textwidth]{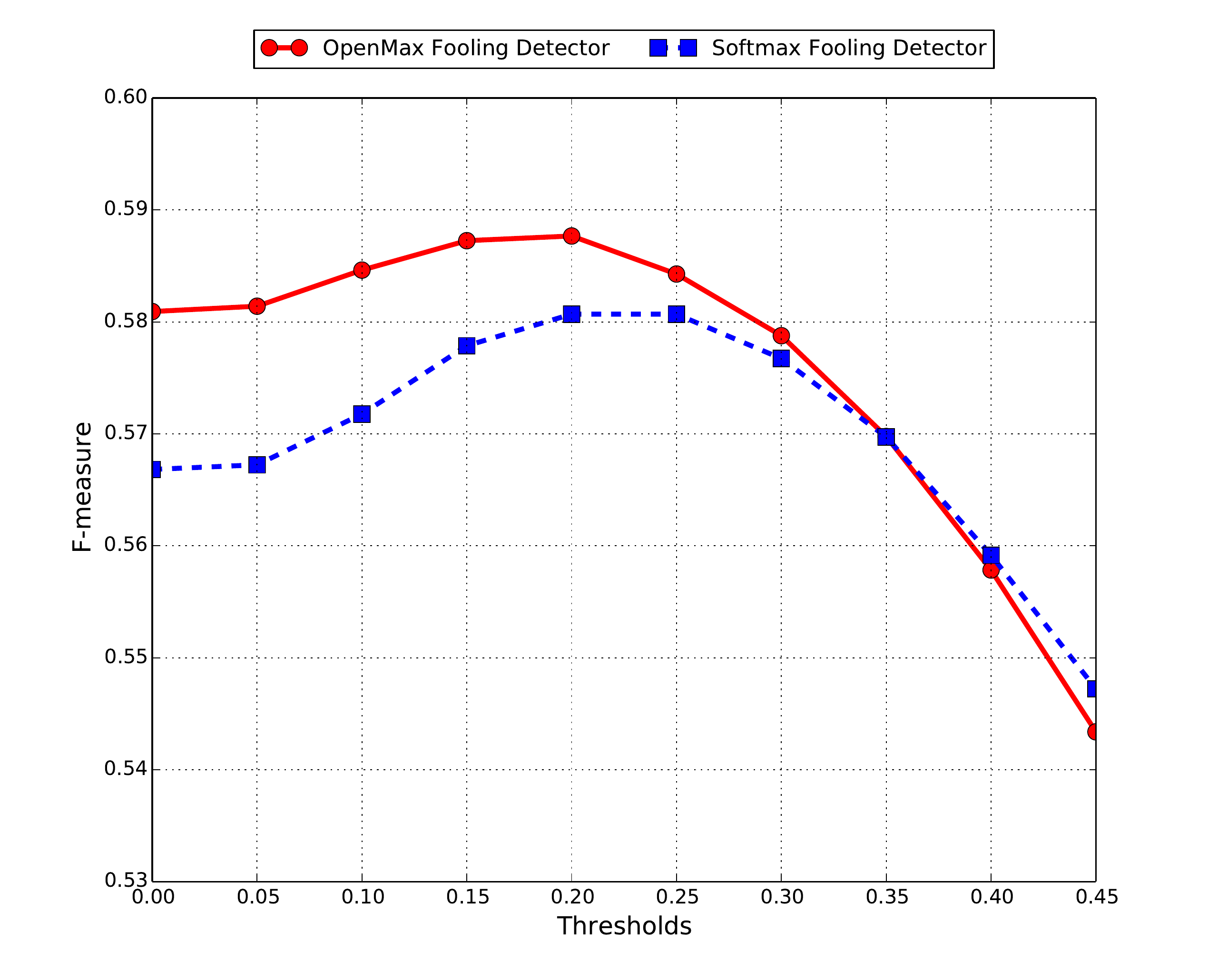}
                \caption{Tail Size 30}
        \end{subfigure}
        ~
        \begin{subfigure}[b]{0.3\textwidth}
                \includegraphics[width=\textwidth]{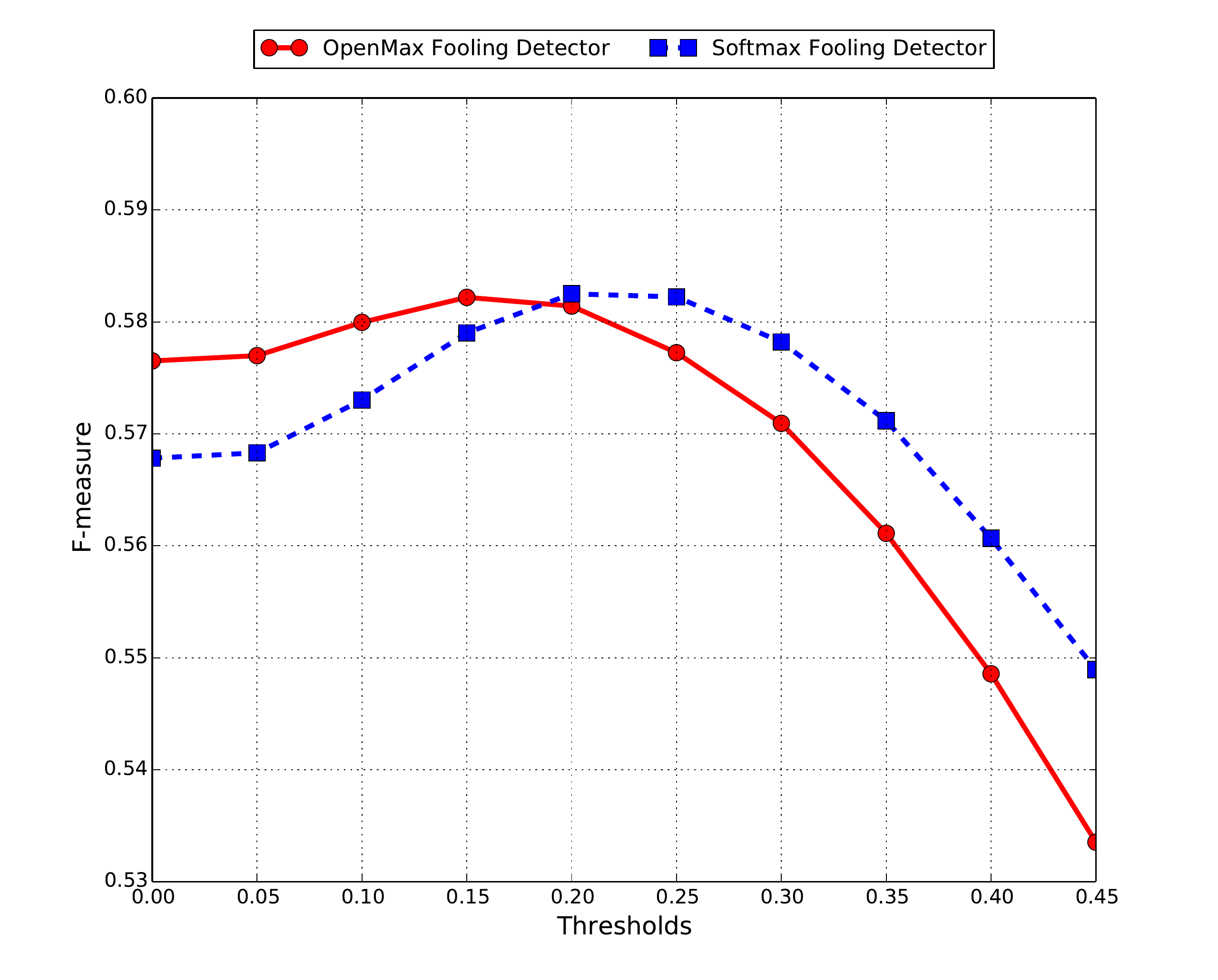}
                \caption{Tail Size 40}
        \end{subfigure}
        ~
        \begin{subfigure}[b]{0.3\textwidth}
                \includegraphics[width=\textwidth]{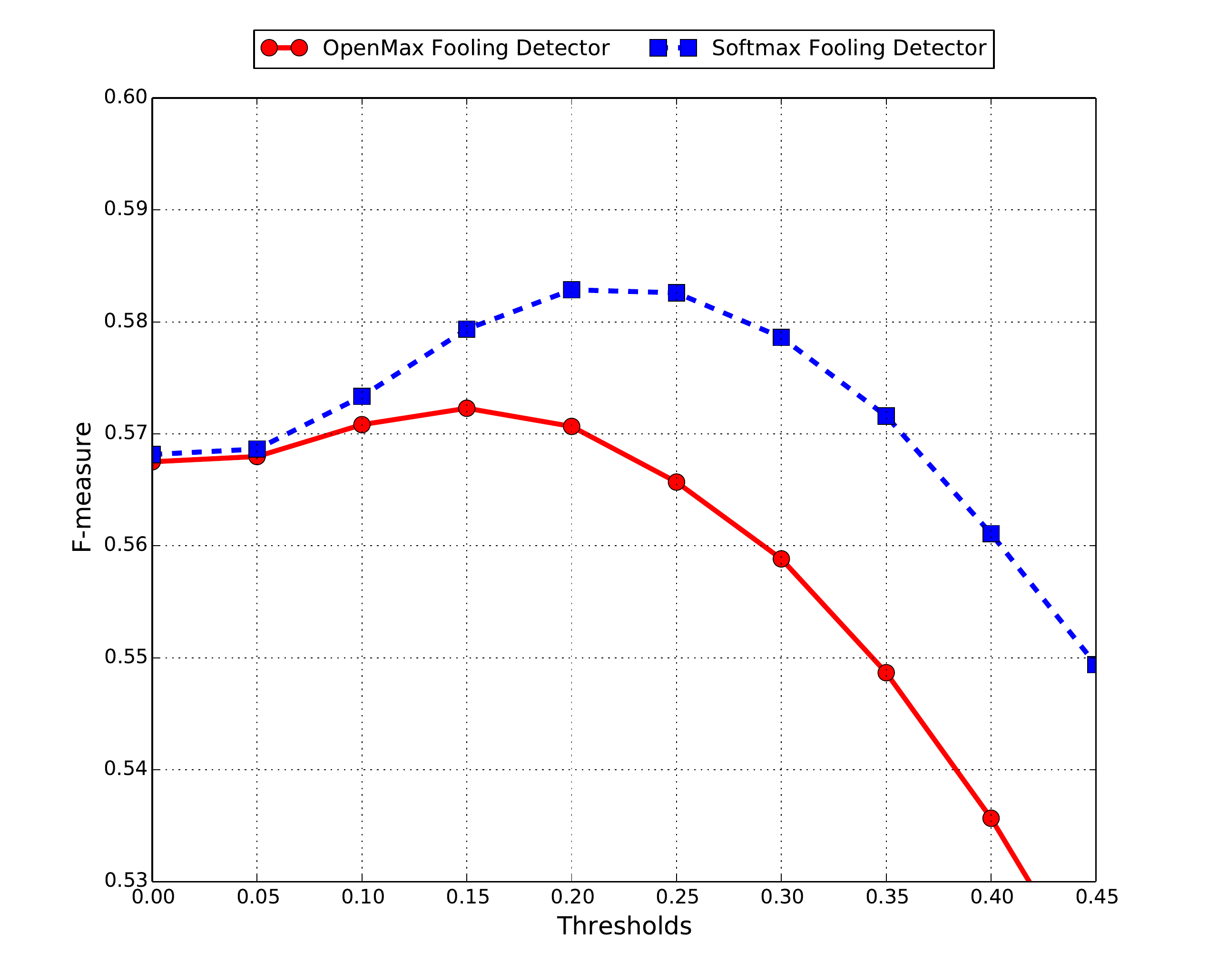}
                \caption{Tail Size 50}
        \end{subfigure}
        ~
        \caption{The graphs shows F-Measure performance of OpenMax and Softmax
          with Open Set testing (using validation, fooling and open set images
          for testing). Each graph shows F-measure plotted against varying
          uncertainty threshold values. Tail size varies in different plots.
          OpenMax reaches its optimal performance at tail size 20. For tail
          sizes larger than 20, though OpenMax becomes good at rejecting images
          from fooling set and open set (Fig ~\ref{fig:tailsizes}), it also
          rejects true images thus reducing accuracy on validation set. Hence,
          we choose tail size 20 for our experiments in main paper.}
         
        \label{fig:tailsizes-fmeasures}
\end{figure*}

\begin{figure*}
        \centering
        \begin{subfigure}[b]{0.3\textwidth}
                \includegraphics[width=\textwidth]{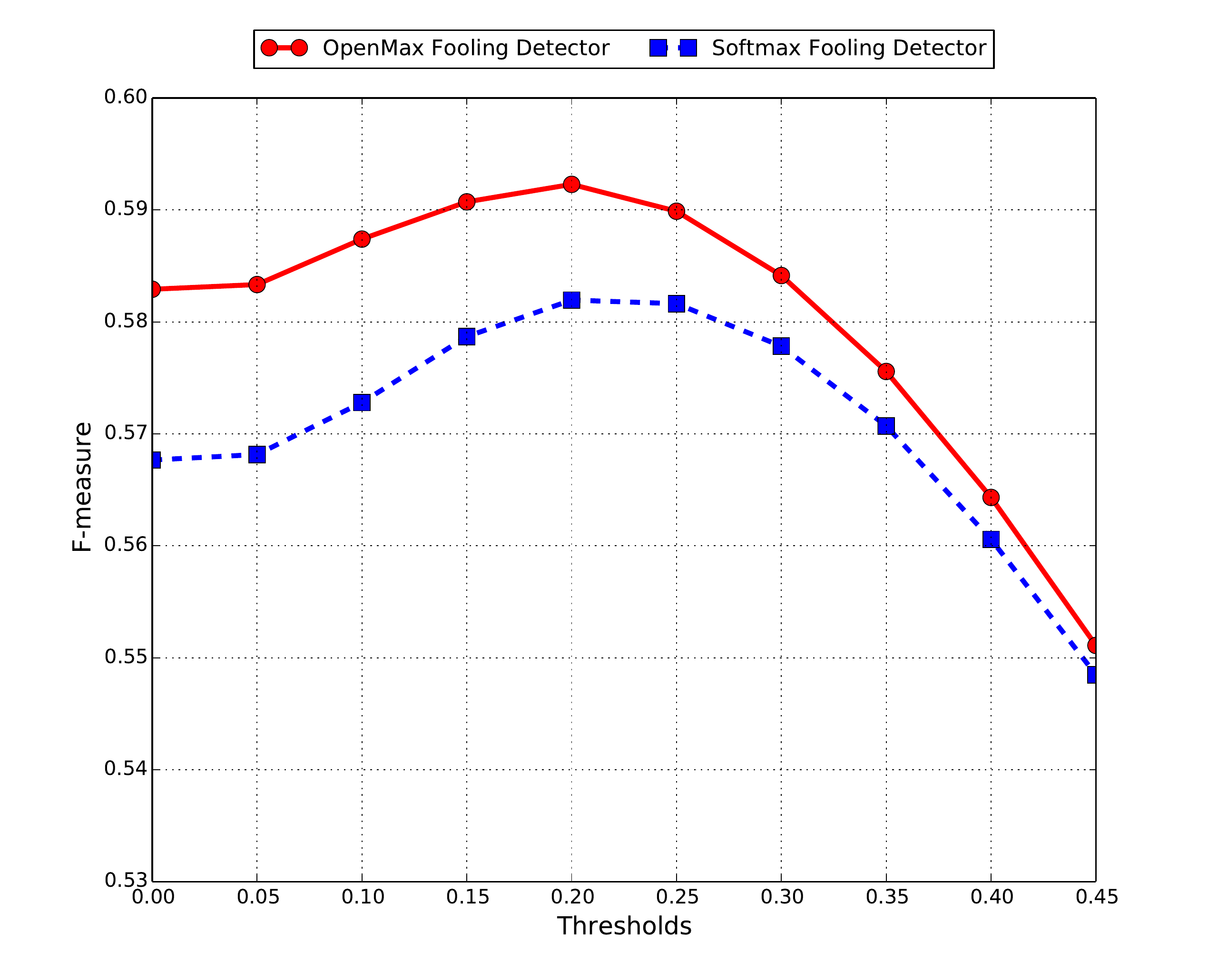}
                \caption{Tail Size 20, Alpha Rank 5}
                                 \vspace{20pt}
        \end{subfigure}%
        ~ %add desired spacing between images, e. g. ~, \quad, \qquad, \hfill etc.
          %(or a blank line to force the subfigure onto a new line)
        \begin{subfigure}[b]{0.3\textwidth}
                \includegraphics[width=\textwidth]{fmeasure_tailsize_20_alpha_rank_10_distance_type_eucos_all.pdf}
                \caption{Tail Size 20, Alpha Rank 10 (optimal)}
                                 \vspace{20pt}
        \end{subfigure}
           ~
        \caption{The above figure shows performance of OpenMax and Softmax as number of top classes to be considered for recalibrating
        are changed. In our experiments, we found best performance when top 10 classes (i.e. $\alpha=10$) were considered for recalibration. }
        \label{fig:alpha-rank-fmeasure}
\end{figure*}

\begin{figure*}
        \centering
        \begin{subfigure}[b]{0.3\textwidth}
                \includegraphics[width=\textwidth]{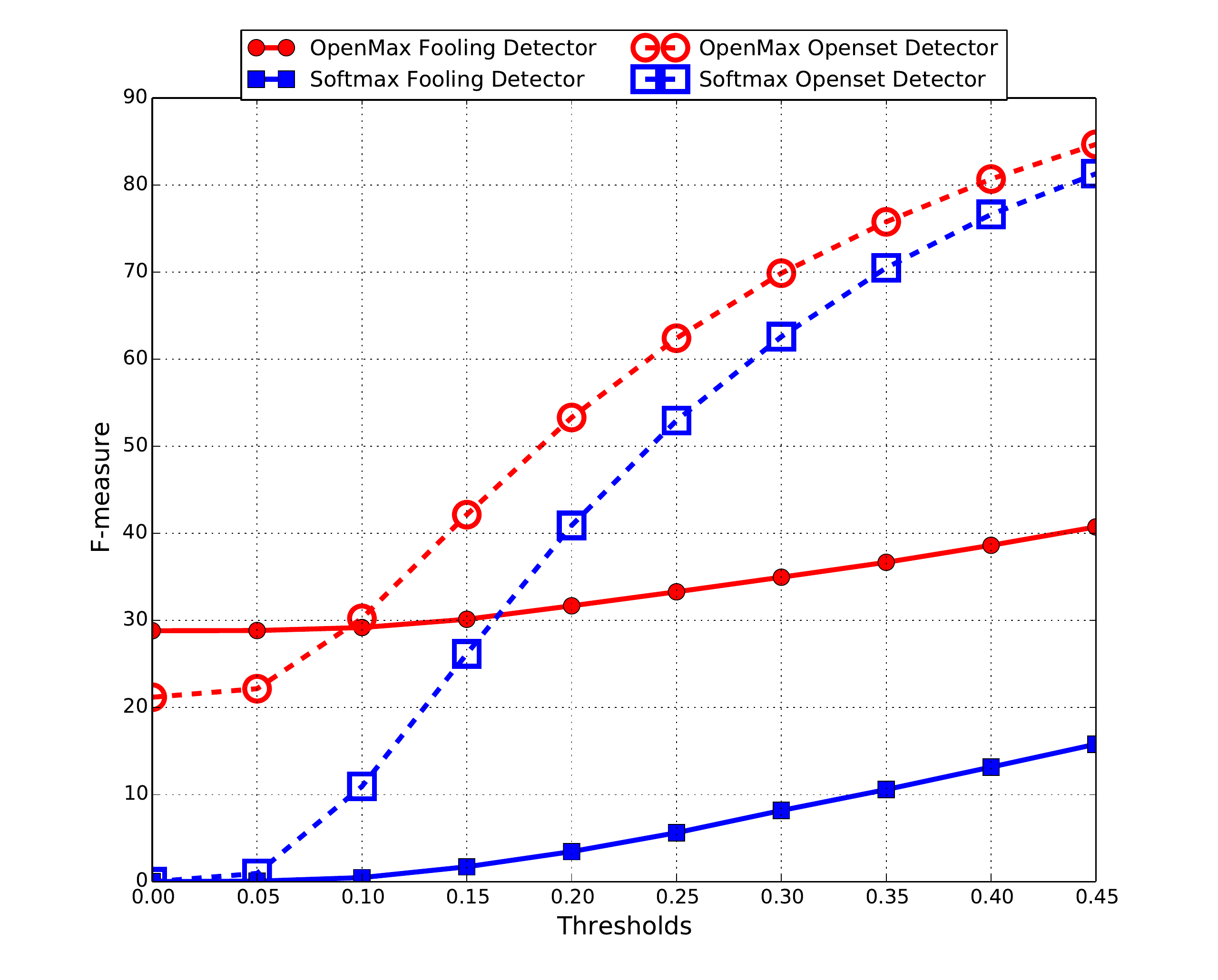}
                \caption{Tail Size 20, Alpha Rank 5}
                                 \vspace{20pt}
                \label{fig:hog-closed}
        \end{subfigure}%
        ~ %add desired spacing between images, e. g. ~, \quad, \qquad, \hfill etc.
          %(or a blank line to force the subfigure onto a new line)
        \begin{subfigure}[b]{0.3\textwidth}
                \includegraphics[width=\textwidth]{tailsize_20_alpha_rank_10_distance_type_eucos_all.pdf}
                \caption{Tail Size 20, Alpha Rank 10 (optimal)}
                                 \vspace{20pt}
        \end{subfigure}
           ~
        \caption{The figure shows fooling detection and open set detection accuracy for varying $alpha$ sizes. In our experiments,
        $alpha$ rank of 10 yielded best results. Increasing $alpha$ value beyond 10 did not result in any performance gains.}
         
        \label{fig:alpha-rank}
\end{figure*}

\begin{figure*}
        \centering
        \begin{subfigure}[b]{0.3\textwidth}
                \includegraphics[width=\textwidth]{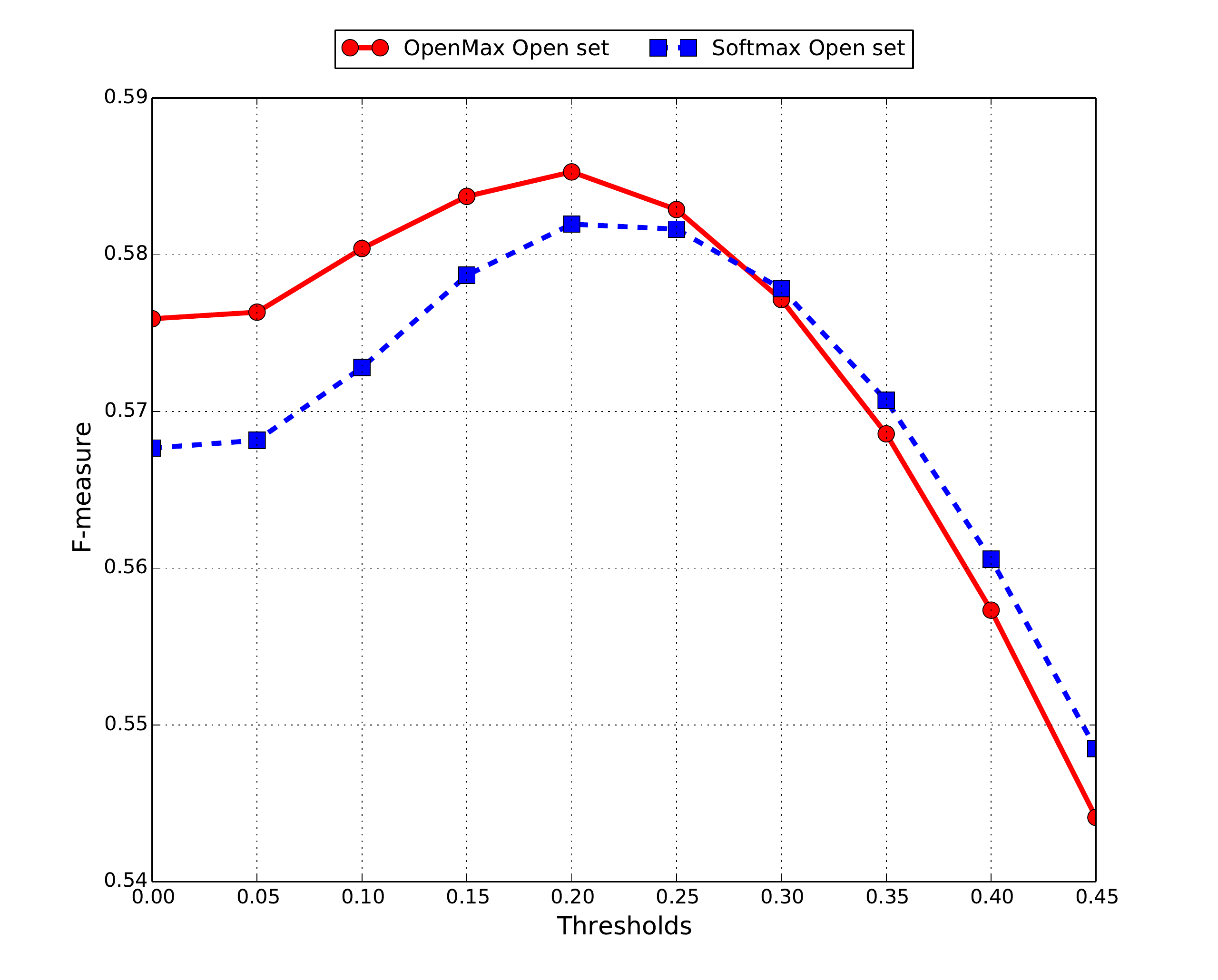}
                \caption{Cosine Distance, Tail Size 20, Alpha Rank 10}
                                 \vspace{20pt}
        \end{subfigure}%
        ~ %add desired spacing between images, e. g. ~, \quad, \qquad, \hfill etc.
          %(or a blank line to force the subfigure onto a new line)
        \begin{subfigure}[b]{0.3\textwidth}
                \includegraphics[width=\textwidth]{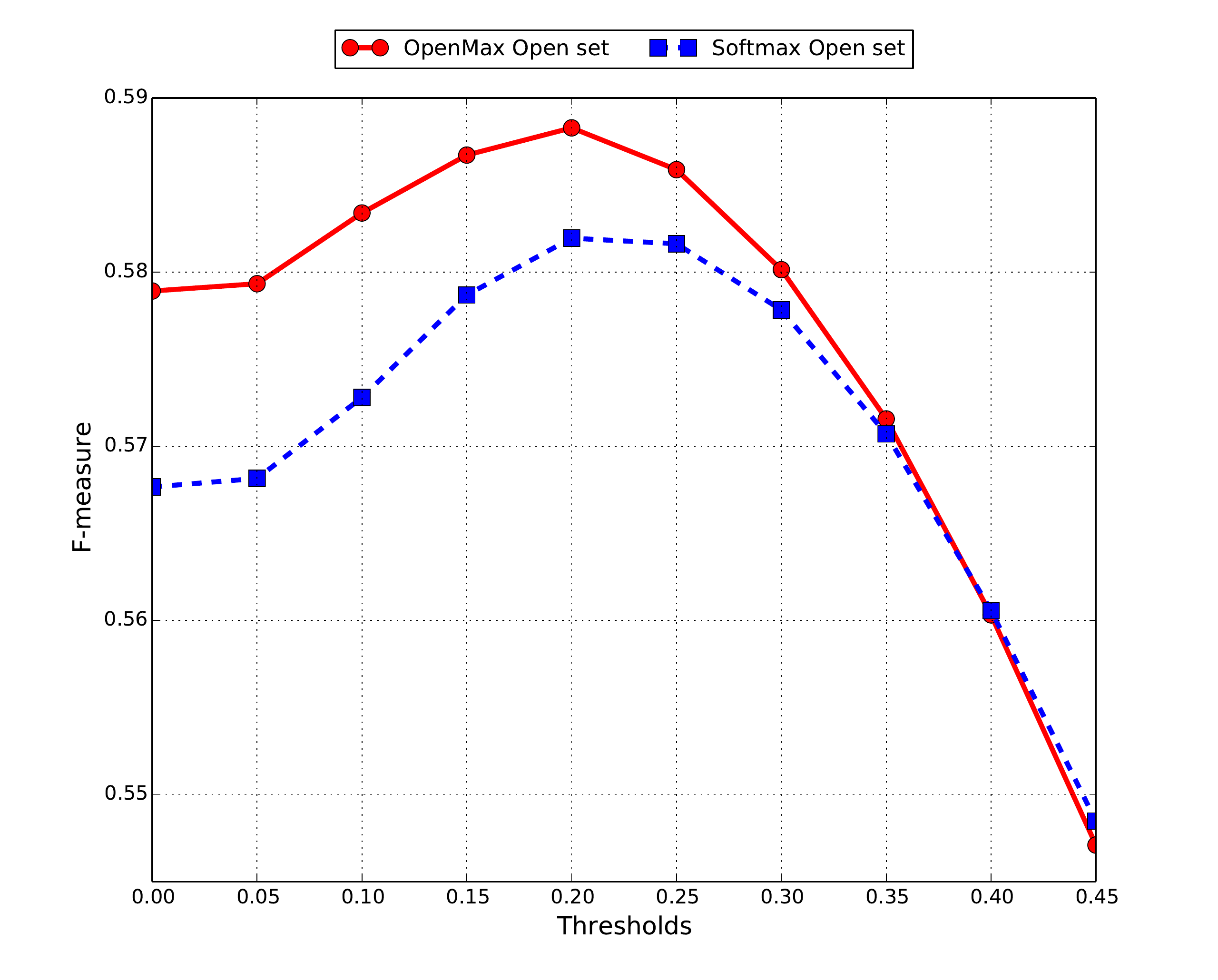}
                \caption{Euclidean Distance, Tail Size 20, Alpha Rank 10}
                                 \vspace{20pt}
        \end{subfigure}
           ~           
             \begin{subfigure}[b]{0.3\textwidth}
                \includegraphics[width=\textwidth]{fmeasure_tailsize_20_alpha_rank_10_distance_type_eucos_all.pdf}
                \caption{Euclidean-Cosine distance, Tail Size 20, Alpha Rank 10 (optimal) Note scale difference!}
                                 \vspace{20pt}                                 
        \end{subfigure}
        ~
        \caption{The above figure shows performance of OpenMax and Softmax for different types of distance measures. We found
        the performance trend to be similar, with euclidean-cosine distance performing best. }
        \label{fig:distance-fmeasures}
\end{figure*}

\subsection{Top Classes to be considered for revision $\alpha$}

In Alg 2 of the main paper, we present a methodology to calibrate FC8 scores via OpenMax. In this process, we also incorporate a process to adjust class probability as well as estimating the probability for the unknown unknown class. For this purpose, in Alg 2 (main paper), we consider ``top'' classes to
revise (line 2, Alg 2, main paper), which is controlled by parameter $\alpha$. We call this parameter as $\alpha$ rank, where value
of $\alpha$ suggests total number of ``top'' classes to revise. In our experiments we found that optimal performance is obtained when
$\alpha=10$. At lower values of $\alpha$ we see drop in F-Measure performance. If we continue to increase $\alpha$ values beyond
10, we see almost no gain in F-Measure performance or fooling/open set detection accuracy. The most likely reason for this lack of change
in performance beyond $\alpha=10$ is lower ranked classes have very small FC8 activations and do not provide any significant change
in OpenMax probability. The results for varying values of $\alpha$ are presented in Figs ~\ref{fig:alpha-rank-fmeasure} and 
~\ref{fig:alpha-rank}.

\subsection{Distance Measures}

We tried different distance measures to compute distances between Mean
Activation Vectors and Activation Vector of an incoming test image. We tried
cosine distance, euclidean distance and euclidean-cosine distance. Cosine
distance and euclidean distances compared marginally worse compared to
euclidean-cosine distance. Cosine distance does not provide for a compact
abating property hence may not restrict open space for points that have small
degree of separation in terms of angle but still far away in terms of euclidean
distance. Euclidean-cosine distance finds the closest points in a hyper-cone,
thus restricting open space and finding closest points to Mean Activation
Vector. Euclidean distance and euclidean-cosine distance performed very similar
in terms of performance. In Fig ~\ref{fig:distance-fmeasures} we show effect of
different distances on over all performance. We see that OpenMax still performs
better than SoftMax, and euclidean-cosine distance performs the best of those
tested.

\section{Qualitative Examples}

It is often useful to look at qualitative examples of success and failure.
Fig.~\ref{fig:openset-wrong} -- Fig.~\ref{fig:openset-wrong2} shows examples where OpenMax failed to detect
open set examples. Some of these were from classess in ILSVRC 2010 that were
close but not identical to classes in ILSVRC 2012. Other examples are objects
from distinct ILSVRC 2010 classes that were visually very similar to a
particular object class in ILSVRC 2012. Finally we show an example where OpenMax
processed a ILSVRC 2012 validation image but reduced its probability thus Caffe
with SoftMax provides the correct answer but OpenMax gets this example wrong.

\section{Confusion Map of Mean Activation Vectors}
Because detection/rejection of unknown classes depends on the distance mean
activation vector (MAV) of the highest scoring FC8 classes. Note this is
different from finding the distance from the input to the cloest MAV. However, we
still find that for unknown classes that are only fine-grain variants of known
classes, the system will not likely reject them. Similarly for adversarial images,
if an image is adversarially modified to a ``neary-by'' is is much less likely the
OpenMax will reject/detect it. Thus it is useful to consider the confusion
between existing classes.

\section{Comparison with the 1-vs-set algorithm. }
The main paper focused on direct extensions within the Deep Networks. While we
consider it tangential, reviewers might worry that applying other models, e.g. a
linear based 1-vs-set open set algorithm\cite{openset-pami13} to the FC8 data
would provide better results. For completeness we did run these experiments. We
used liblinear to train a linear SVM on the training samples from the 1000
classes. We also trained a 1-vs-set machine using the liblinear extension cited
in \cite{openworld_2015}, refining it on the training data for the 1000 classes.
The 1-Vs-Set algorithm achieves an overall F-measure of only .407, which is much
lower than the .595 of the OpenMax approach.

\begin{figure*}
          \begin{subfigure}[b]{\columnwidth}
            {  \includegraphics[width=\columnwidth]{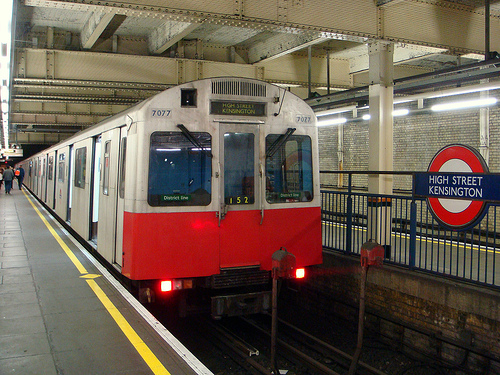}}
          \end{subfigure}
          \begin{subfigure}[b]{\columnwidth}
            {  \includegraphics[width=\columnwidth]{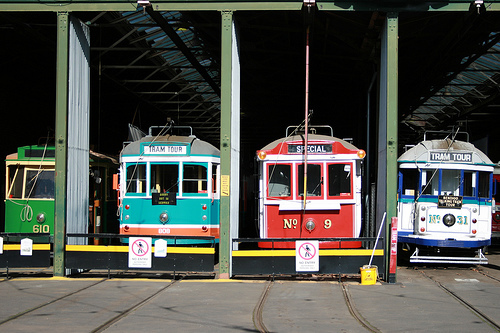}}
            \end{subfigure}
                \caption{Left is an Image from ILSVRC 2010, ``subway train'', n04349306. OpenMax and Softmax both classify as n04335435.
                Instead of `unknown''. OpenMax predicts that the image on left belongs to category
          ``n04335435:streetcar, tram, tramcar, trolley, trolley car'' from  ILSVRC 2012 with an output probability of 0.6391 (caffe probability
                0.5225).  Right is an example image image from ILSVRC 2012, ``streetcar,
                  tram, tramcar, trolley, trolley car'', n04335435 It is easy to
                  see such mistakes are bound to happen since open set classes
                  from ILSVRC 2010 may have have many related categories which
                  have different names, but which are semantically or visually
                  are very similar.  This is why fooling rejection is much stronger than open set rejection.}
                \label{fig:openset-wrong}
\end{figure*}

\begin{figure*}
        \begin{subfigure}[b]{\columnwidth}
              \centerline{\includegraphics[width=.8\columnwidth]{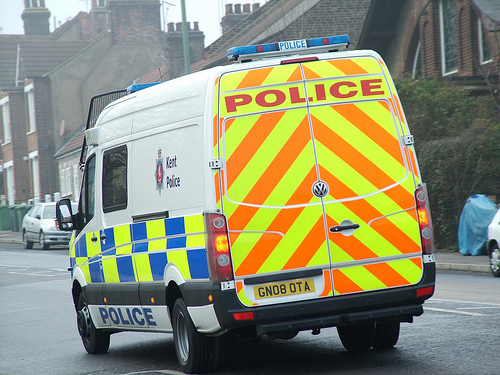}}
                \caption{ A validation example of Openmax Failure/    Softmax labeles it correctly as n03977966 (Police van/police wagon) with probability 0.6463, while openmax incorrect labels it n02701002  (Ambulance) with probability 0.4507. 
) }
                                 \vspace{20pt}
        \end{subfigure}\quad
                \begin{subfigure}[b]{\columnwidth}
                \centerline{\includegraphics[width=.8\columnwidth]{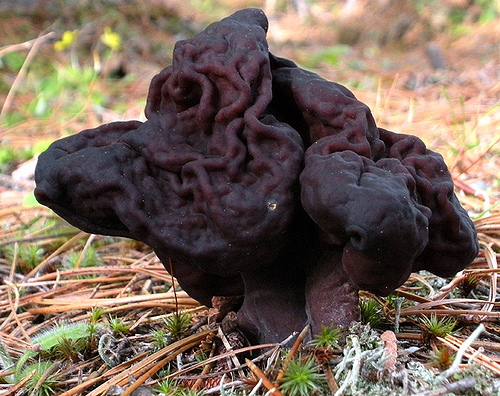}}
                \caption{ Another validation failure example, where softmax
                  classifies it as n13037406 with probability 0.9991, while
                  OpenMax rejects it as unknown. n13037406 is a gyromitra, which
                  is genus of mushroom.}
                                 \vspace{20pt}
        \end{subfigure}
        \caption{The above figure shows an examples of validation image misclassification by
          OpenMax algorithm.         }
        \label{fig:openset-wrong2}
\end{figure*}

\begin{figure*}
\centering
 \includegraphics[width=\textwidth]{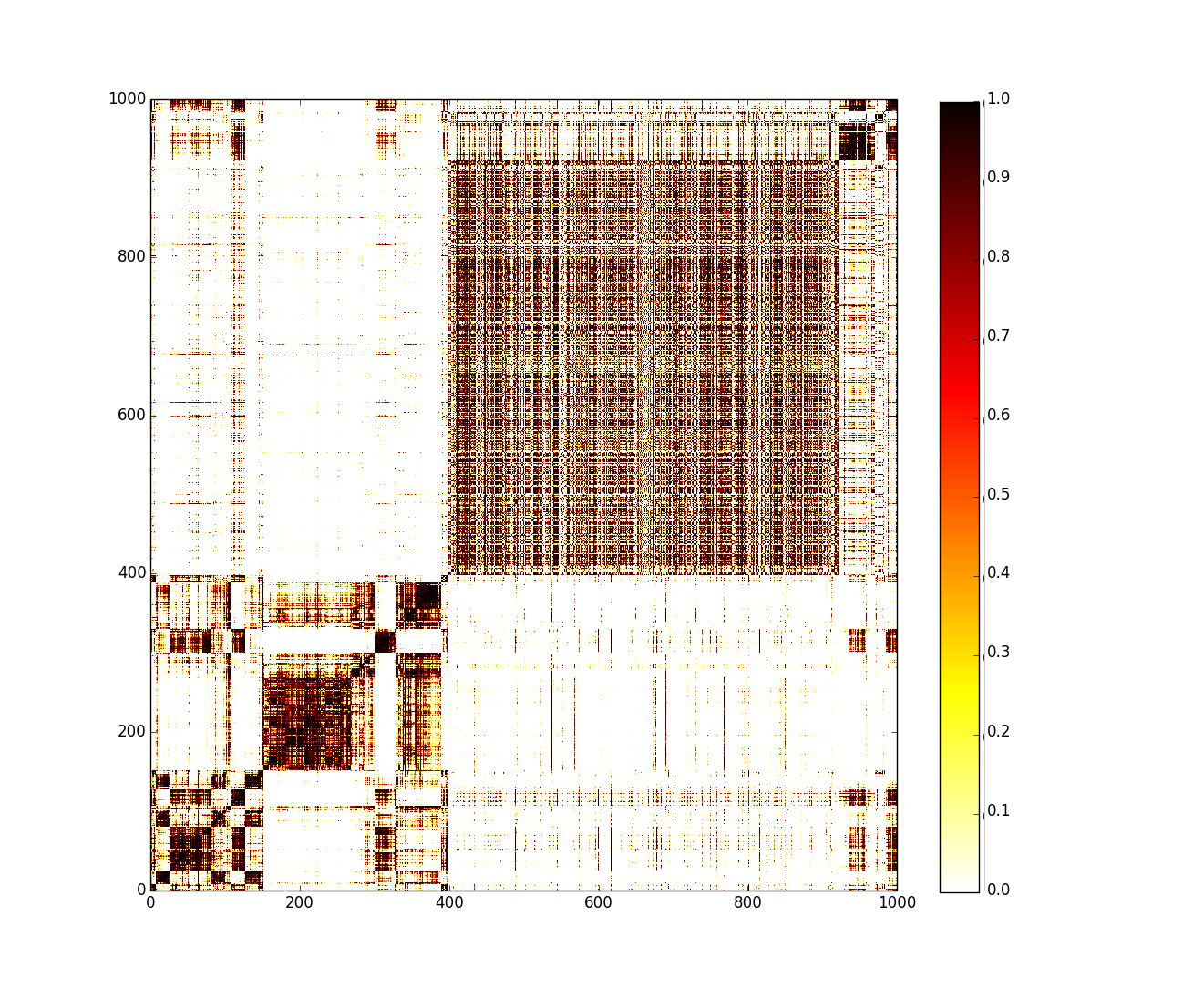}
 \caption{The above figure shows confusion matrix of  distances between Mean Activation Vector (MAV) for each class in ILSVRC 2012
 with MAV of every other class. Lower values of  distances indicate that MAVs for respective classes are very close to each other,
 and higher values of distances indicate classes that are far apart. Majority of misclassifications for OpenMax happen in
 fine-grained categorization, which is to be expected.}
\end{figure*}